%% file: neurips_2026.tex
\documentclass{article}

\usepackage[preprint]{neurips_2026}

\usepackage[utf8]{inputenc} 
\usepackage[T1]{fontenc}    
\usepackage{hyperref}       
\usepackage{url}            
\usepackage{booktabs}       
\usepackage{amsfonts}       
\usepackage{nicefrac}       
\usepackage{microtype}      
\usepackage{xcolor}         
\usepackage{graphicx}       
\usepackage{subcaption}     
\usepackage{wrapfig}
\usepackage{tikz}
\usetikzlibrary{arrows.meta,positioning,fit,backgrounds,calc,decorations.pathreplacing}
\usepackage{comment}

\usepackage{amsmath,amssymb,amsthm,mathtools}
\usepackage{bm}

\hypersetup{
  colorlinks=true,
  linkcolor=blue,
  citecolor=blue,
  urlcolor=blue
}

\input{math_commands}

\newtheorem{theorem}{Theorem}
\newtheorem{lemma}{Lemma}
\newtheorem{proposition}{Proposition}

\newtheorem{assumption}{Assumption}
\newtheorem{remark}{Remark}
\newcommand{\cau}{\mathrm{cau}}
\newcommand{\id}{\mathrm{id}}
\newcommand{\opt}{\mathrm{opt}}

\renewcommand{\softmax}{\mathrm{softmax}}
\newcommand{\vphi}{\bm{\varphi}}

\title{How Does Attention Help? Insights from Random Matrices on Signal Recovery from Sequence Models}

\author{%
  Mohamed El Amine Seddik \\
  AI Research Center, Technology Innovation Institute (TII) \\
  Abu Dhabi, UAE \\
  \texttt{mohamed.seddik@tii.ae} \\
}

\begin{document}

\maketitle

\begin{abstract}
We study the spectral properties of sample covariance matrices constructed from pooled sequence representations, where token embeddings are drawn from a fixed two-class Gaussian mixture table and pooled via (fixed) attention weights. Working in the high-dimensional regime $d,V,N\to\infty$ with $d/V\to\delta$ and $d/N\to\gamma$, we derive exact characterizations of the limiting eigenvalue distribution, outlier eigenvalues, and eigenvector alignment with the hidden signal. The bulk spectrum follows a non-Marchenko--Pastur law given by the free multiplicative convolution $\kappa(\mathrm{MP}_\delta\boxtimes\mathrm{MP}_\gamma)$, reflecting the finite vocabulary structure. Signal recovery undergoes two successive BBP-type phase transitions characterized by the scalars: $\delta,\gamma,\alpha=\vw^\top\mR\vw$ and $\kappa=\|\vw\|^2$, where $\vw$ denotes the attention pooling weights and $\mR$ the positional correlation matrix. An aftermath of our analysis demonstrates that the optimal attention weights maximizing the signal-to-noise ratio $\alpha/\kappa$ are given by the (normalized) top eigenvector of~$\mR$, and we show (as a particular case of our analysis) that parameter-free causal self-attention with $\tau/d$ score scaling yields deterministic harmonic weights that improve signal recovery over mean pooling whenever early tokens carry more signal. Extensive simulations confirm sharp agreement between theory and finite-dimensional experiments.
\end{abstract}

\section{Introduction}\label{sec:intro}

The attention mechanism~\citep{vaswani2017attention} is the central building block of modern transformer architectures that have achieved remarkable success across language modeling~\citep{devlin2019bert,brown2020language}, vision~\citep{dosovitskiy2020image}, and many current AI applications. Despite this empirical success, a precise theoretical understanding of \emph{why} and \emph{how} attention helps remains limited. Most existing theoretical analyses focus on expressivity~\citep{yun2019transformers}, optimization landscapes~\citep{sahiner2022unraveling}, or rank collapse phenomena~\citep{dong2021attention}. A complementary and largely unexplored direction is to study attention through the lens of \emph{statistical signal recovery}: given a structured data model, how does the choice of attention weights affect the ability to recover a latent signal from high-dimensional observations?

Random matrix theory (RMT) provides a natural framework for this question. The well known BBP phase transition within the random matrix community~\citep{baik2005phase,baik2006eigenvalues} characterizes when a hidden signal becomes detectable from the spectrum of a sample covariance matrix, and the associated eigenvector overlap formulas~\citep{paul2007asymptotics,benaych2011eigenvalues} quantify the alignment between the estimated and true signal directions. These tools have been successfully applied to high-dimensional statistics~\citep{johnstone2001distribution,johnstone2009consistency}, wireless communications~\citep{couillet2011random}, and more recently to understanding the high-dimensional behaviour of machine learning algorithms and deep learning representations~\citep{pennington2017nonlinear,seddik2020random,loureiro2021learning,couillet2022random}.

In this work, we propose a tractable statistical model for sequence data with a \emph{fixed} two-class Gaussian mixture embedding table and positional correlations, and we study how attention-based pooling affects signal recovery. Our main contributions are:
\textbf{(1)}~the exact limiting spectral distribution of the sample covariance of the pooled representations in the high-dimensional regime, characterized by a non-Marchenko--Pastur law (Theorem~\ref{thm:bulk});
\textbf{(2)}~two successive BBP-type phase transitions for signal detectability, at the population and sample levels, with explicit outlier eigenvalue and eigenvector overlap formulas (Theorems~\ref{thm:pop_bbp}--\ref{thm:sample_overlap});
\textbf{(3)}~a general optimality result showing that the attention weights maximizing signal recovery are the (normalized) top eigenvector of $\mR$ (Theorem~\ref{thm:optimal});
\textbf{(4)}~as a particular case of our general framework, a concrete analysis of parameter-free causal self-attention, yielding deterministic harmonic weights that provably improve $\alpha/\kappa$ over mean pooling when the signal is carried on early sequence positions (Section~\ref{sec:causal}); and
\textbf{(5)}~extensive simulations validating all theoretical predictions (Section~\ref{sec:experiments}).

\paragraph{Related work.}
The Baik-Ben Arous-Péché (BBP) phase transition was established by~\citet{baik2005phase} for complex spiked Wishart matrices and extended by~\citet{baik2006eigenvalues}. Eigenvector overlap formulas were derived by~\citet{paul2007asymptotics} and generalized by~\citet{benaych2011eigenvalues,benaych2012singular}. The Marchenko--Pastur law~\citep{marvcenko1967distribution} and free probability~\citep{nica2006lectures,mingo2017free} underpin our bulk analysis, with comprehensive treatments in~\citet{bai2010spectral,couillet2022random}. On the ML side,~\citet{pennington2017nonlinear} initiated RMT for deep learning. Theoretical analyses of attention span expressivity~\citep{yun2019transformers}, rank collapse~\citep{dong2021attention}, inductive biases~\citep{edelman2022inductive}, and kernel limits~\citep{hron2020infinite}. Closer to our focus on the statistical role of softmax attention,~\citet{duranthon2025statistical} prove a statistical advantage of softmax attention over linear alternatives in a single-location regression setting, while~\citet{boncoraglio2026singleheadattentionhighdimensions} study the generalization, weights spectra and scaling laws of a single-head attention layer in high dimensions using tools closely related to ours. A complementary line of work~\citep{jha2025random} analyzes the interaction between trained attention and statistical structure in a high-dimensional spiked-data setup. The latent-variable perspective on embeddings by~\citet{arora2016latent} further motivates our spiked embedding model. To the best of our knowledge, our work is the first to provide an exact RMT analysis of how attention weights modulate signal recovery under the considered setup of a \textit{fixed} embedding table.

\section{Sequence Statistical Model}\label{sec:statistical_model}

\subsection{Two-class embedding table}\label{sec:embed}
Let $\vmu\in \sR^d$ be a deterministic vector with direction $\vu := \vmu / \Vert \vmu \Vert$. We consider a two-class \textit{fixed} embedding table $\mE = [\ve_1, \ldots, \ve_V] \in \sR^{d\times V}$ with
\begin{align}\label{eq:embedding}
    \ve_v = s_v \vmu + \vz_v, \quad v\in [V],
\end{align}
where $(s_v)_{v=1}^V \in \{ \pm 1 \}^V$ are \textit{balanced} signs ($V/2$ entries $+1$ and $V/2$ entries $-1$), $\vz_v\sim \gN(\vzero, \mI_d)$ are i.i.d.\@ noise vectors, and $\mZ = [\vz_1, \ldots, \vz_V] \in \sR^{d\times V}$. The embedding table is a two-class Gaussian mixture of means $\pm \vmu$ with $\Vert \vmu\Vert$ controlling class separability.

\begin{remark}[Relaxing Gaussianity]\label{rem:universality}
The Gaussian assumption on the noise vectors $\vz_v$ is made for convenience and can be substantially relaxed. Standard universality results in random matrix theory~\citep{bai2010spectral,couillet2022random} imply that the bulk spectrum of $\mS$ and the location of its outlier eigenvalue are insensitive to the precise distribution of the $\vz_v$'s, provided each coordinate has zero mean, unit variance, and a bounded fourth-order moment. All our main results therefore extend straightforwardly to this broader class of embedding-noise distributions.
\end{remark}

\subsection{Sequence data}\label{sec:seq_data}
We are given $N$ i.i.d.\@ sequences of length $T$. For each sample $n\in [N]$ and position $t\in [T]$, draw $\xi_{n,t} \in \{ \pm 1 \}$ such that $\E\left[ \xi_n \xi_n^\top \right] = \mR$ for some positional correlation matrix $\mR \in \sR^{T\times T}$, and draw tokens $x_{n,t}$ uniformly from $\{v:s_v = \xi_{n,t}\}$. The embedded sequence is $\mX^{(n)} = [ X_1^{(n)}, \ldots,  X_T^{(n)} ]$ with $ X_t^{(n)}:=\ve_{x_{n,t}}$, and we study the sample covariance $\mS := \frac1N \mC \mC^\top \in \sR^{d\times d}$, where $\mC = [\vc_1, \ldots, \vc_N]\in \sR^{d\times N}$ filled column-wise with the pooled representations
\begin{align}\label{eq:pooling}
    \vc_n = \sum_{t=1}^T w_t X_t^{(n)}, \quad \sum_{t=1}^T w_t=1, \quad \vw:=(w_1, \ldots, w_T) \in \sR^T.
\end{align}
We study under which conditions on the attention vector $\vw$ one can achieve optimal recovery of the latent signal $\vu$ from $\mS$. For simplicity, $\vw$ is uniform across sequences (independent of $n$), as $\mR$ is also independent of $n$. Even this ``simple'' setting yields a nontrivial spectral theory.

\begin{assumption}[Growth Rates]\label{ass:growth}
    $d,V,N\to \infty$ with $d/V\to \delta\in (0, \infty)$, $d/N \to \gamma\in (0, \infty)$, $T=O(1)$, $\Vert \vmu \Vert = O(1)$.
\end{assumption}

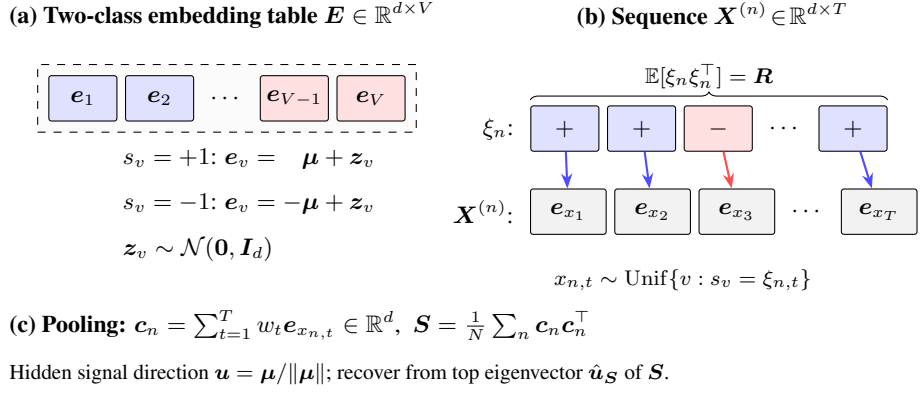
\begin{figure}[t!]
    \centering
    \begin{tikzpicture}[
        font=\small,
        >={Stealth[length=2mm,width=1.5mm]},
        embedbox/.style={rectangle, draw, minimum width=9mm, minimum height=6mm, rounded corners=1pt, inner sep=1pt},
        plusclass/.style={embedbox, fill=blue!12},
        minusclass/.style={embedbox, fill=red!12},
        tokenbox/.style={rectangle, draw, minimum width=10mm, minimum height=6mm, rounded corners=1pt, inner sep=1pt, fill=gray!10},
        arrow/.style={->, thick},
        every node/.style={align=center}
    ]
    
    \node[anchor=west] at (0,4.4)
    {\textbf{(a) Two-class embedding table $\mE\in\sR^{d\times V}$}};
    
    \matrix (A) at (3,3.3) [row sep=6mm, column sep=1mm] {
    \node[plusclass]  (e1) {$\ve_1$}; &
    \node[plusclass]  (e2) {$\ve_2$}; &
    \node              (edots) {$\cdots$}; &
    \node[minusclass] (eV1) {$\ve_{V-1}$}; &
    \node[minusclass] (eV)  {$\ve_V$}; \\
    };
    
    \node[anchor=west] at (1.5,2.5)
    {$s_v=+1$: $\ve_v=\phantom{-}\vmu+\vz_v$};
    
    \node[anchor=west] at (1.5,1.9)
    {$s_v=-1$: $\ve_v=-\vmu+\vz_v$};

    \node[anchor=west] at (1.5,1.3)
    {$\vz_v\sim\gN(\vzero,\mI_d)$};
    
    \begin{scope}[on background layer]
    \node[draw, dashed, inner sep=4pt, fit=(e1)(eV), fill=gray!3] {};
    \end{scope}
    
    \node[anchor=west] at (7.5,4.4)
    {\textbf{(b) Sequence $\mX^{(n)}\!\in\!\sR^{d\times T}$}};
    
    \matrix (B1) at (9.0,2.9) [row sep=6mm, column sep=1mm] {
    \node (xilab) {$\xi_n$:}; &
    \node[plusclass]  (xi1) {$+$}; &
    \node[plusclass]  (xi2) {$+$}; &
    \node[minusclass] (xi3) {$-$}; &
    \node              (xid) {$\cdots$}; &
    \node[plusclass]  (xiT) {$+$}; \\
    };
    
    \matrix (B2) at (9.0,1.8) [row sep=6mm, column sep=1mm] {
    \node (xlab) {$\mX^{(n)}$:}; &
    \node[tokenbox] (x1) {$\ve_{x_1}$}; &
    \node[tokenbox] (x2) {$\ve_{x_2}$}; &
    \node[tokenbox] (x3) {$\ve_{x_3}$}; &
    \node           (xdd) {$\cdots$}; &
    \node[tokenbox] (xT) {$\ve_{x_T}$}; \\
    };
    
    \draw[decorate,decoration={brace,amplitude=3pt,raise=2pt}]
    (xi1.north west) -- (xiT.north east)
    node[midway,above=3pt,font=\footnotesize]
    {$\E[\xi_n\xi_n^\top]=\mR$};
    
    \draw[arrow, blue!70] (xi1) -- (x1);
    \draw[arrow, blue!70] (xi2) -- (x2);
    \draw[arrow, red!70]  (xi3) -- (x3);
    \draw[arrow, blue!70] (xiT) -- (xT);
    
    \node[anchor=west] at (7.2,0.9)
    {\footnotesize $x_{n,t}\sim\mathrm{Unif}\{v:s_v=\xi_{n,t}\}$};
    
    \node[anchor=west] at (0.0,0.3)
    {\textbf{(c) Pooling:} $\vc_n=\sum_{t=1}^T w_t \ve_{x_{n,t}}\in\sR^d$, \
    $\mS=\tfrac1N\sum_n \vc_n\vc_n^\top$};
    
    \node[anchor=west] at (0.0,-0.35)
    {\footnotesize Hidden signal direction $\vu=\vmu/\Vert\vmu\Vert$; recover from top eigenvector $\hat\vu_\mS$ of $\mS$.};
    
    \end{tikzpicture}
    
    \caption{Illustration of the sequence statistical model from Section~\ref{sec:statistical_model}. \textbf{(a)} A fixed two-class embedding table: $V/2$ vocabulary items have sign $+1$ (blue, mean $+\vmu$), the other $V/2$ have sign $-1$ (red, mean $-\vmu$); each token embedding is centered on its class mean plus isotropic Gaussian noise. \textbf{(b)} Each sequence samples a correlated sign vector $\xi_n\in\{\pm1\}^T$ with $\E[\xi_n\xi_n^\top]=\mR$ and then draws, independently at each position, a random token from the matching class. \textbf{(c)} Pooling with weights $\vw$ and stacking across samples gives the sample covariance $\mS$; signal recovery is measured by the alignment of the top eigenvector $\hat\vu_\mS$ of $\mS$ with the hidden signal direction $\vu$.}
    \label{fig:model}
\end{figure}

\subsection{Practical justification of the statistical model}\label{sec:justification}
Figure~\ref{fig:model} summarizes the three components of the statistical data model: the two-class embedding table, the correlated positional signs that select tokens within their class, and the pooled representation $\vc_n$ yielding the sample covariance $\mS$ under analysis. While our model is deliberately considered for analytical tractability, each component has a practical motivation. \textit{(i) Two-class embeddings:} Pre-trained word embeddings exhibit a low-rank-plus-noise structure~\citep{arora2016latent} (see Appendix~\ref{app:gpt2} for an empirical illustration on GPT-2), and the two-class assumption is the minimal nontrivial case for studying signal recovery. \textit{(ii) Fixed embedding table:} In many practical settings the embedding table is pre-trained and frozen during downstream tasks, making the ``fixed table'' regime realistic. The key consequence, that the table randomness introduces a second source of spectral noise captured by $\delta = d/V$, disappears only in the impractical limit $V\gg d$. \textit{(iii) Positional correlations:} The matrix $\mR$ captures the empirical observation that semantic content is unevenly distributed across positions, which aligns with typical applications (e.g. language). \textit{(iv) Sequence-independent weights:} While full self-attention produces data-dependent weights, the $\tau/d$ scaling regime we analyze (Section~\ref{sec:causal}) shows that attention weights \emph{concentrate} to deterministic values as $d\to\infty$, justifying the sequence-independent assumption as a first-order approximation.

\vspace{0.25em}
With the statistical model fixed, we now turn to its spectral analysis: the next section derives the limiting bulk distribution of $\mS$, characterizes the two BBP-type phase transitions that govern signal recovery, and expresses the alignment $(\hat\vu_\mS^\top \vu)^2$ in closed form in terms of the scalars $(\alpha,\kappa,\delta,\gamma)$.

\section{Main Results}\label{sec:main}

\subsection{Pooled covariance decomposition}\label{sec:cov_decomp}
The columns of $\mC$ are i.i.d.\@ (conditionally on $\mZ$) with covariance
\begin{align}\label{eq:Sigma}
    \mSigma := \E\left[ \vc \vc^\top \mid \mZ  \right] = \alpha \Vert \vmu \Vert^2 \vu \vu^\top + \kappa \mSigma_{\mZ}, \quad \mSigma_{\mZ}:= \frac{1}{V} \mZ \mZ^\top,
\end{align}
where the two key scalars are
\begin{align}\label{eq:alpha_kappa}
    \alpha(\vw):= \vw^\top \mR \vw \quad \text{and} \quad \kappa(\vw) := \Vert \vw \Vert^2.
\end{align}
The scalar $\alpha$ captures the effective signal strength after pooling and $\kappa$ controls the noise level. The effective signal-to-noise ratio (SNR) is $\rho := \alpha\|\vmu\|^2/\kappa$, and optimal signal recovery is achieved by maximizing $\alpha/\kappa$ as we will demonstrate subsequently.

\subsection{Limiting spectrum of $\bf S$}\label{sec:bulk}

\begin{theorem}[Bulk density]\label{thm:bulk}
    Under Assumption~\ref{ass:growth}, the limiting spectral density (LSD) of $\mS$ is the free multiplicative convolution $\kappa \left( \mathrm{MP}_\delta \boxtimes \mathrm{MP}_\gamma \right)$ whose Stieltjes transform $m(z)$ satisfies
    \begin{equation}\label{eq:cubic}
\delta\gamma \kappa\,z^2 m(z)^3
-\kappa z(\delta+\gamma-2\delta\gamma)\,m(z)^2
-\big(z+\kappa(\delta-1)(1-\gamma)\big)\,m(z)
-1
=0.
\end{equation}
The eigenvalue density is $\rho(x) = \frac{1}{\pi}\lim_{\eta\downarrow0}\Im\,m(x+i\eta)$. The companion transform $\underline{m}(z):=-\frac{1-\gamma}{z} + \gamma m(z)$ corresponds to the Gram matrix $\mC^\top \mC/N$.
\end{theorem}
\begin{proof}
The proof is deferred to Appendix~\ref{app:cubic}.
\end{proof}

The LSD of $\mS$ differs from the standard Marchenko--Pastur law~\citep{marvcenko1967distribution} unless $\delta=0$ (i.e., infinite vocabulary relative to the dimension of embeddings). The finite vocabulary introduces a multiplicative noise structure: $\mS$ is the product of two independent Wishart-type components (the embedding table randomness, governed by $\delta$, and the sampling randomness, governed by $\gamma$), whose interplay is captured by the free multiplicative convolution~\citep{nica2006lectures}. For attention, the key insight is that the \emph{shape} of the bulk depends only on $(\delta,\gamma)$, which are fixed by the problem dimensions, while the \emph{scale} is $\kappa(\vw) = \|\vw\|^2$. Thus attention affects the bulk only through $\kappa$: non-uniform weights inflate $\kappa$ above the mean-pooling value $1/T$, widening the bulk.

\begin{proposition}[Right bulk edge]\label{prop:edge}
    Let $q=\delta\gamma$ and $r_0=\delta + \gamma + \delta \gamma$. Define $\Delta_0 = r_0(r_0^3 + 216 q^2)$, $\Delta_1 = 2(r_0^6 - 540 q^2 r_0^3 - 5832 q^4)$, $\phi = \arccos \left( \Delta_1/(2\Delta_0^{3/2}) \right)$. The right edge is $\lambda_+:= \kappa(\vw) \max\{x_0, x_1, x_2\}$ where
    \begin{align*}
        x_k = - \frac{1}{12q} \left( r_0^2 - 12 qr_0 + 12q^2 - 12q + 2 \sqrt{\Delta_0} \cos\!\left( \frac{\phi + 2\pi k}{3} \right) \right), \quad k=0,1,2.
    \end{align*}
\end{proposition}
\begin{proof}
The proof is deferred to Appendix~\ref{app:cubic}.
\end{proof}

The right edge $\lambda_+$ scales linearly with $\kappa(\vw)$. Since the BBP phase transition depends on whether the spike exceeds $\lambda_+$, increasing $\kappa$ raises the detectability threshold. This formalizes the intuition that ``noisier pooling makes signal detection harder''.

\subsection{Spike behavior of the population covariance}\label{sec:spike_pop}

\begin{theorem}[Population BBP transition]\label{thm:pop_bbp}
    Define $\rho:= \alpha \Vert \vmu \Vert^2 / \kappa$. The top eigenvalue of $\mSigma$ exhibits a BBP phase transition: if $\rho\leq \delta + \sqrt{\delta}$, no outlier exists and the top eigenvalue sticks to $\beta_+:= \kappa (1 + \sqrt{\delta})^2$; if $\rho> \delta + \sqrt{\delta}$, an outlier appears at
        \begin{align}\label{eq:beta_out}
            \beta_\mathrm{out} := \frac{\kappa\rho(\rho-\delta +1)}{\rho - \delta}.
        \end{align}
    The top eigenvector $\hat \vu_\mSigma$ of $\mSigma$ has overlap $\vert \hat{\vu}_\mSigma^\top \vu \vert^2 \to 1 - \delta/(\rho-\delta)^2$ for $\rho> \delta + \sqrt{\delta}$, and $0$ otherwise.
\end{theorem}
\begin{proof}
The proof is deferred to Appendix~\ref{app:pop}.
\end{proof}

This result reveals a fundamental interplay between vocabulary size (through $\delta$) and signal recovery. When $\delta > 0$, the finite vocabulary introduces a \emph{strictly positive} detection threshold $\rho > \delta + \sqrt{\delta}$: even with infinite data ($\gamma\to 0$), one cannot recover the signal if $\alpha\|\vmu\|^2/\kappa$ is too small. For attention, the critical observation is that \emph{increasing $\alpha/\kappa$ lowers the effective threshold}. At the critical point $\rho = \delta + \sqrt{\delta}$, the overlap jumps from $0$ to a positive value as a sharp phase transition.

\subsection{Spike behavior of the sample covariance}\label{sec:spike_sample}

\begin{theorem}[Sample outlier eigenvalue]\label{thm:sample_spike}
    In the supercritical regime, let $\beta := \beta_\mathrm{out}$. If $\beta$ exceeds $\beta_\mathrm{crit} := -1/\underline{m}(\lambda_+)$, then $\mS$ has an outlier $\lambda_\mathrm{out} > \lambda_+$ given by the largest root of
    \begin{equation}\label{eq:lambda_quad}
        \delta\kappa\,\lambda^2 + \left(-\beta^2\gamma + \beta\kappa(\gamma(1+\delta) - 2\delta)\right)\lambda + \beta^2\left(\beta\gamma + \kappa(1-\gamma)(\delta - \gamma)\right) = 0.
    \end{equation}
\end{theorem}
\begin{proof}
The proof is deferred to Appendix~\ref{app:sample_spike}.
\end{proof}

\begin{theorem}[Sample eigenvector alignment]\label{thm:sample_overlap}
    When $\lambda_\mathrm{out}$ exists, the top eigenvector $\hat\vu_\mS$ of $\mS$ satisfies
    $|\hat\vu_\mS^\top \hat\vu_\mSigma|^2 = 1/(\beta\,\lambda_\mathrm{out}\,\underline{m}'(\lambda_\mathrm{out}))$,
    where $\underline{m}'$ is computed via implicit differentiation of~\eqref{eq:cubic} (see Appendix~\ref{app:sample_overlap}).
\end{theorem}
\begin{proof}
The proof is deferred to Appendix~\ref{app:sample_overlap}.
\end{proof}

The sample-level phase transition introduces a \emph{second} detection barrier due to finite $N$. Even if the population covariance has an outlier ($\rho > \delta + \sqrt{\delta}$), the sample covariance may fail to exhibit a separated spike if $N$ is too small. The total alignment of $\hat\vu_\mS$ with the hidden signal $\vu$ factorizes as
\begin{align}\label{eq:total_alignment}
    |\hat\vu_\mS^\top \vu|^2 \approx \underbrace{|\hat\vu_\mS^\top \hat\vu_\mSigma|^2}_{\text{sample noise}} \cdot \underbrace{|\hat\vu_\mSigma^\top \vu|^2}_{\text{vocabulary noise}} = \frac{1}{\beta\,\lambda_\mathrm{out}\,\underline{m}'(\lambda_\mathrm{out})} \left(1 - \frac{\delta}{(\rho - \delta)^2}\right).
\end{align}
The first factor captures degradation from finite samples ($\gamma > 0$) and the second from finite vocabulary ($\delta > 0$). Attention improves \emph{both} factors through $\alpha/\kappa$: a larger ratio widens the distance between the bulk edge $\lambda_+$ and the outlier $\lambda_{\text{out}}$, therefore yielding better signal recovery.

\subsection{Phase transition thresholds}\label{sec:phase}
Recalling $\rho = \alpha\|\vmu\|^2/\kappa$, the population BBP condition translates to $\|\vmu\| > \mu_\mathrm{pop}(\vw) := \sqrt{\kappa(\delta + \sqrt{\delta})/\alpha}$. The sample threshold $\mu_\mathrm{samp}(\vw)$ is obtained by solving $\beta_\mathrm{out}(\rho) = \beta_\mathrm{crit}$, yielding a quadratic in $\rho$ (details in Appendix~\ref{app:phase}). Both thresholds decrease when $\alpha/\kappa$ increases, confirming that maximizing the effective SNR through $\vw$ is the key to improving signal detectability and recovery.

\vspace{0.25em}
The detectability thresholds of this section depend on $\vw$ exclusively through the single scalar $\alpha(\vw)/\kappa(\vw)$. For optimality, we aim to characterize $\vw$ that maximizes this ratio which translates to a simple spectral optimization problem that we solve in closed form in Section~\ref{sec:optimal} below.

\section{Optimal Attention Weights}\label{sec:optimal}

The following result characterizes the attention weights $\vw$ that maximize signal recovery.

\begin{theorem}[Optimal attention weights]\label{thm:optimal}
Let $\mR$ be the (symmetric) positional correlation matrix, with largest eigenvalue $\lambda_{\mathrm{max}}(\mR)$ and corresponding eigenspace $\gE_{\mathrm{max}}$. The optimization problem
\begin{align}\label{eq:opt_problem}
    \max_{\vw\in\sR^T} \frac{\alpha(\vw)}{\kappa(\vw)} = \frac{\vw^\top \mR \vw}{\vw^\top \vw} \quad \text{subject to} \quad \vone^\top \vw = 1
\end{align}
satisfies: \textbf{(a)} the supremum equals $\lambda_{\mathrm{max}}(\mR)$; \textbf{(b)} a maximizer exists if and only if $\vone$ is not orthogonal to $\gE_{\mathrm{max}}$, in which case
\begin{align}\label{eq:w_opt}
    \vw^\opt = \frac{\vv_{\mathrm{max}}}{\vone^\top \vv_{\mathrm{max}}}, \quad \text{where} \quad \vv_{\mathrm{max}} \in \gE_{\mathrm{max}} \text{ with } \vone^\top \vv_{\mathrm{max}} \neq 0,
\end{align}
and $\alpha(\vw^\opt)/\kappa(\vw^\opt) = \lambda_{\mathrm{max}}(\mR)$.
\end{theorem}
\begin{proof}
The proof is deferred to Appendix~\ref{app:optimal}.
\end{proof}

This result reveals that the best attention weights for signal recovery are determined entirely by the spectral structure of the positional correlation matrix $\mR$. The constraint $\sum_t w_t = 1$ does not reduce the achievable SNR because the Rayleigh quotient $\alpha/\kappa$ is scale-invariant. Thus the optimal value $\lambda_{\mathrm{max}}(\mR)$ is the \emph{fundamental limit} on the effective SNR for any pooling strategy. When $\lambda_{\mathrm{max}}$ is simple, $\vw^\opt$ is unique (up to sign normalization) and concentrates weight on the positions that contribute most to the dominant correlation mode.

In particular, for a \emph{prefix model} with $L$ correlated positions and $T-L$ independent ones, $\lambda_{\mathrm{max}}(\mR) = L$ (the top eigenvalue of the block $\vone_L\vone_L^\top + \mI_{T-L}$), achieved when $\vw^\opt$ places all weight on the first $L$ positions. In contrast, mean pooling achieves only $\alpha_\id/\kappa_\id = L^2/T + (T-L)/T$, which can be much smaller.

\paragraph{Phase diagram.}
To visualize the combined effect of the hidden signal strength $\Vert \vmu \Vert$ and the vocabulary ratio $\delta = d/V$ on recovery, we plot the theoretical alignment $(\hat\vu_\mS^\top \vu)^2$ given by the closed form~\eqref{eq:total_alignment} as a heatmap over the $(\Vert \vmu \Vert, \delta)$ plane. We use here the \emph{spiked} correlation model $\mR=\mI_T+\theta_R\,\vu_R\vu_R^\top$, with $\theta_R>0$ and $\vu_R$ a unit vector supported on the first $L$ positions, for which the Rayleigh quotient $\alpha/\kappa$ contrasts strongly between a generic weight vector and the top eigenvector of $\mR$, to highlight the advantage of considering the optimal weights rather than an arbitrary pooling vector. Figure~\ref{fig:phase_diagram} compares mean pooling ($\vw=\vone/T$) with the optimal weights $\vw^\opt$ of Theorem~\ref{thm:optimal}; the dashed white curve is the sample BBP phase-transition boundary $\Vert \vmu \Vert = \mu_\mathrm{samp}(\vw,\delta)$ obtained from~\eqref{eq:beta_out} and~\eqref{eq:lambda_quad}. The recoverable region (bright area) is substantially larger for the optimal weights, and its lower frontier is shifted to smaller $\Vert \vmu \Vert$, illustrating how maximizing the ratio $\alpha(\vw)/\kappa(\vw)$ lowers the detectability threshold uniformly in $\delta$.

\begin{figure}[t!]
    \centering
    \begin{subfigure}[b]{0.46\textwidth}
        \centering
        \includegraphics[width=\textwidth]{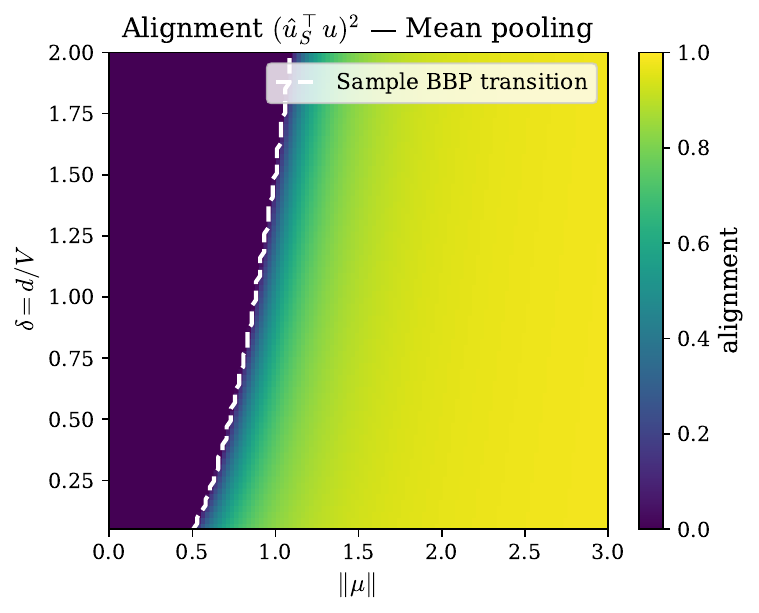}
        \caption{Mean pooling ($\vw=\vone/T$)}
        \label{fig:pd_mean}
    \end{subfigure}
    \hfill
    \begin{subfigure}[b]{0.46\textwidth}
        \centering
        \includegraphics[width=\textwidth]{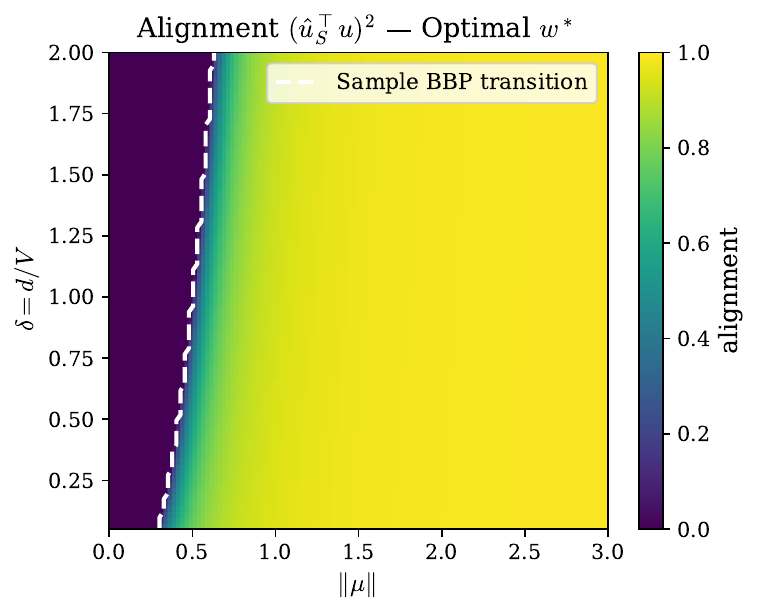}
        \caption{Optimal weights $\vw^\opt$}
        \label{fig:pd_opt}
    \end{subfigure}
    \caption{Phase diagrams of the theoretical alignment $(\hat\vu_\mS^\top \vu)^2$ in the $(\Vert \vmu\Vert,\delta)$ plane, at fixed $\gamma=d/N=0.5$, $T=20$, for the \emph{spiked} positional correlation model $\mR=\mI_T+\theta_R\,\vu_R\vu_R^\top$ with $\theta_R=10$ and $\vu_R$ uniformly supported on the first $L=5$ positions. The heatmap is obtained from the closed-form formula~\eqref{eq:total_alignment}; the dashed white curve is the sample BBP phase-transition threshold $\Vert \vmu\Vert = \mu_\mathrm{samp}(\vw,\delta)$. Under the spiked $\mR$, the optimal weights $\vw^\opt\propto \vu_R$ concentrate mass on the $L$ signal-carrying positions and achieve $\alpha/\kappa=\lambda_{\max}(\mR)=1+\theta_R=11$, while mean pooling only sees the average correlation and pays an $O(T)$ dilution; the recoverable region of the optimal panel is therefore substantially wider, and its lower frontier is shifted to visibly smaller $\Vert\vmu\Vert$ across the whole range of $\delta$.}
    \label{fig:phase_diagram}
\end{figure}

\section{A Particular Case: Causal Self-Attention}\label{sec:causal}

The framework of Section~\ref{sec:optimal} treats the pooling vector $\vw$ as a free parameter. We now instantiate it on a concrete, parameter-free attention mechanism: causal self-attention with score scaling $\tau/d$. 
In this regime the attention weights concentrate to deterministic, sequence-independent values, so Theorems~\ref{thm:bulk}--\ref{thm:optimal} apply with $\vw$ replaced by the limiting causal weight vector derived below, where the scalars $\alpha(\vw)$ and $\kappa(\vw)$ are given in closed forms involving harmonic numbers. 

\begin{lemma}[Deterministic weight limit]\label{lem:det_weights}
    Under Assumption~\ref{ass:growth} with $T$ fixed, causal attention pooling weights converge: $\vw^\mathrm{att} = \vw^{(0)} + O_P(d^{-1/2})$, where $\vw^{(0)}$ uses uniform weights on each admissible key set.
\end{lemma}
\begin{proof}
The proof is deferred to Appendix~\ref{app:det_weights}.
\end{proof}

\begin{proposition}[Harmonic weights]\label{prop:harmonic}
    With harmonic numbers $\mathcal{H}_n := \sum_{k=1}^n 1/k$, the deterministic causal pooled weights are
    \begin{align}\label{eq:w_cau}
        w_1^{(0)} = \frac{1 + \mathcal{H}_{T-1}}{T}, \qquad w_s^{(0)} = \frac{\mathcal{H}_{T-1} - \mathcal{H}_{s-1}}{T}, \quad s=2,\ldots,T.
    \end{align}
\end{proposition}
\begin{proof}
The proof is deferred to Appendix~\ref{app:harmonic}.
\end{proof}

Specifically, these weights are monotonically decreasing: earlier positions receive more weight, with $w_1 \propto 1 + \ln(T-1)$ and $w_T = 0$. This ``early-position bias'' is a direct consequence of the causal mask and is the mechanism through which causal attention improves signal recovery when early tokens are more informative. Causal attention is \emph{not} optimal in general (it does not know $\mR$), but its built-in bias toward early positions makes it a good heuristic when signal concentrates at the beginning of the sequence, as in autoregressive language models.

In particular, for the prefix model~with $\mR = [\vone_L\vone_L^\top, \vzero; \vzero, \mI_{T-L}]$, the causal scalars are $\kappa_\cau = (2T-1+\mathcal{H}_{T-1})/T^2$ and $\alpha_\cau$ is given by~\eqref{eq:alpha_cau_full} in the appendix. Causal attention improves over mean pooling ($\alpha_\cau/\kappa_\cau > \alpha_\id/\kappa_\id$) whenever $L < T$, with a logarithmic gain for $L=1$.

\section{Experiments}\label{sec:experiments}

Having derived the limiting spectrum, the phase-transition thresholds, the optimal weights, and their causal-attention specialization, we now verify each of these predictions on finite-dimensional simulations. We further consider synthetic simulations on a downstream classification task to connect our findings with the ML realm. 

\begin{figure}[t!]
    \centering
    \begin{subfigure}[b]{0.48\textwidth}
        \centering
        \includegraphics[width=\textwidth]{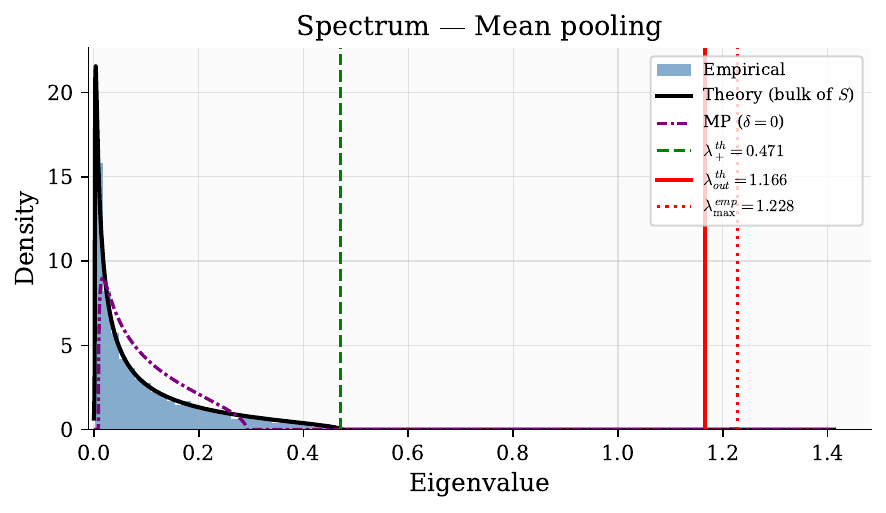}
        \caption{Mean pooling ($\vw = \vone/T$)}
    \end{subfigure}
    \hfill
    \begin{subfigure}[b]{0.48\textwidth}
        \centering
        \includegraphics[width=\textwidth]{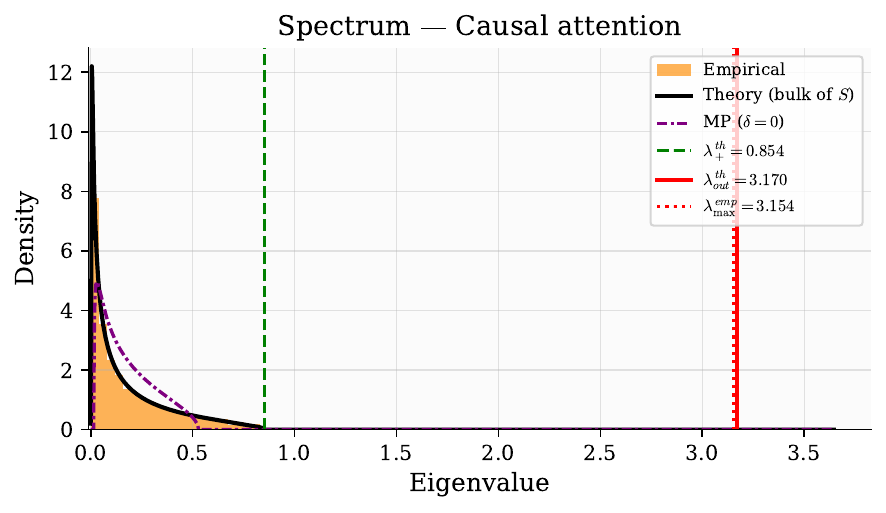}
        \caption{Causal attention ($\vw = \vw^\cau$)}
    \end{subfigure}
    \caption{Empirical vs.\ theoretical eigenvalue density of $\mS$ for $d\!=\!500$, $V\!=\!800$, $N\!=\!1000$, $T\!=\!10$, $L\!=\!3$, $\|\vmu\|\!=\!2.5$. Vertical lines: bulk edge $\lambda_+$ (green), theoretical outlier (red solid), empirical top eigenvalue (red dotted). The black curve is the theoretical bulk from Theorem~\ref{thm:bulk}, and the purple dash-dot curve is the scaled Marchenko--Pastur law (corresponding to $\delta\!=\!0$, i.e.\ infinite vocabulary): its support is visibly narrower, highlighting the extra spread induced by the finite-vocabulary factor $\mathrm{MP}_\delta$ in the free multiplicative convolution. The causal bulk is wider ($\kappa_\cau > \kappa_\id$) but the outlier is much more separated, confirming the improved SNR.}
    \label{fig:bulk}
\end{figure}

\paragraph{Bulk spectrum (Figure~\ref{fig:bulk}).} The theoretical bulk density from Theorem~\ref{thm:bulk} matches the empirical histograms precisely for both pooling strategies. Comparing panels, causal attention has a wider bulk (larger $\kappa$) but a dramatically more separated outlier ($\lambda_\mathrm{out}^\mathrm{cau} = 3.17$ vs $\lambda_\mathrm{out}^\mathrm{id} = 1.17$), confirming that the SNR gain $\alpha_\cau/\kappa_\cau > \alpha_\id/\kappa_\id$ translates to better signal separation.

\begin{figure}[t!]
    \centering
    \begin{subfigure}[b]{0.48\textwidth}
        \centering
        \includegraphics[width=\textwidth]{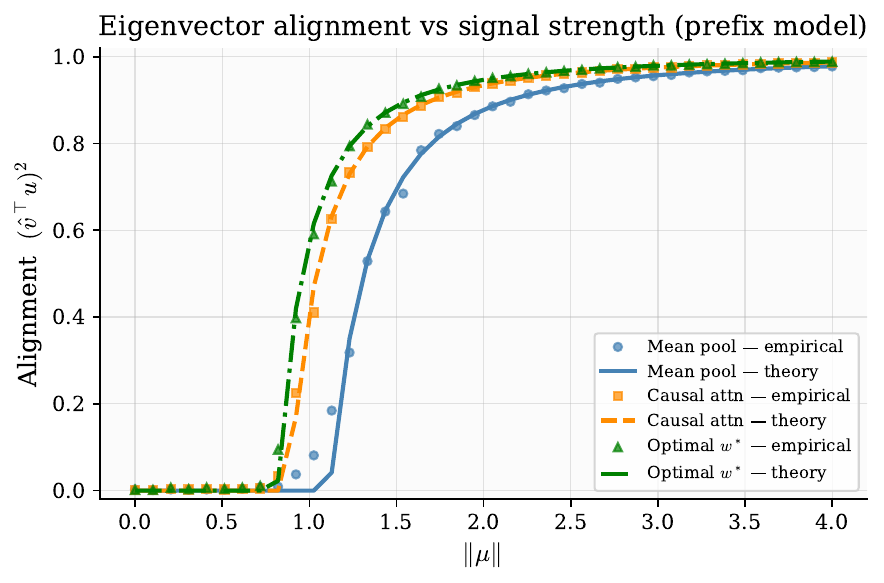}
        \caption{Prefix model ($T\!=\!10$, $L\!=\!3$)}
        \label{fig:align_prefix}
    \end{subfigure}
    \hfill
    \begin{subfigure}[b]{0.48\textwidth}
        \centering
        \includegraphics[width=\textwidth]{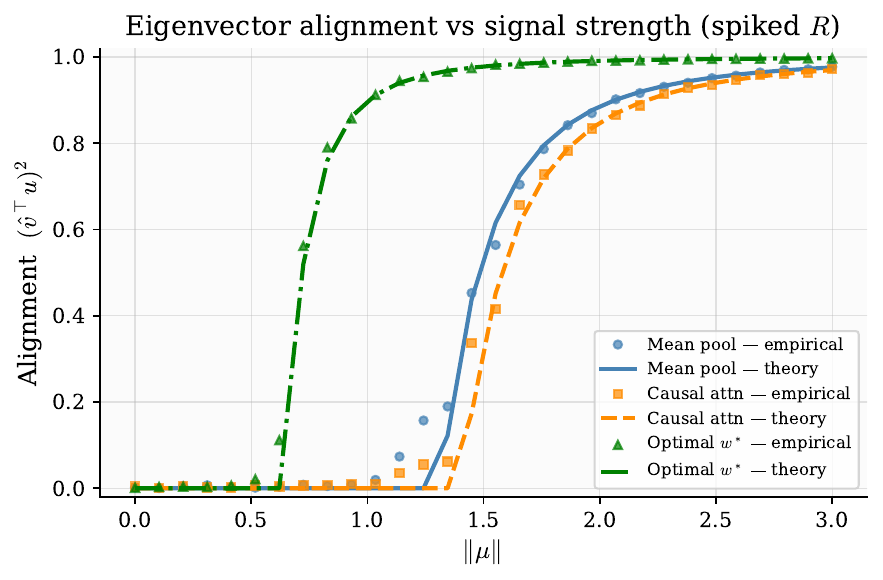}
        \caption{Spiked correlation ($T\!=\!20$, spike$\,=10$)}
        \label{fig:align_spiked}
    \end{subfigure}
    \caption{Eigenvector alignment $|\hat\vu_\mS^\top\vu|^2$ vs.\ $\|\vmu\|$ for three pooling strategies. Lines: theory~\eqref{eq:total_alignment}; markers: empirical averages. \textbf{(a)} Prefix model with optimal weights achieving the earliest transition (Theorem~\ref{thm:optimal}). \textbf{(b)} Spiked $\mR$ as described in Section \ref{sec:experiments}: the optimal weights (green) significantly outperform both mean pooling and causal attention, which here have almost similar performance since the signal is spread across multiple early positions.}
    \label{fig:alignment}
\end{figure}

\paragraph{Alignment with optimal weights (Figure~\ref{fig:alignment}).} We compare three strategies: mean pooling, causal attention, and optimal weights $\vw^\opt$ (Theorem~\ref{thm:optimal}). Figure~\ref{fig:align_prefix} uses the prefix model ($T=10$, $L=3$): the optimal weights achieve the earliest phase transition, followed by causal attention, then mean pooling, which is consistent with $\lambda_{\mathrm{max}}(\mR_S) > \alpha_\cau/\kappa_\cau > \alpha_\id/\kappa_\id$. Figure~\ref{fig:align_spiked} uses a spiked correlation matrix $\mR = \mI_T + \theta_R \vu_R\vu_R^\top$, $\theta_R = 10$ and $T=20$ where $\vu_R$ has nonzero entries only in positions $\{1,\ldots,5\}$ with a sign pattern. Here the optimal weights, which adapt to the specific structure of $\vu_R$, achieve dramatically better alignment than both mean pooling and causal attention, illustrating the result of Theorem~\ref{thm:optimal}.

\begin{figure}[t!]
    \centering
    \begin{subfigure}[b]{0.48\textwidth}
        \centering
        \includegraphics[width=\textwidth]{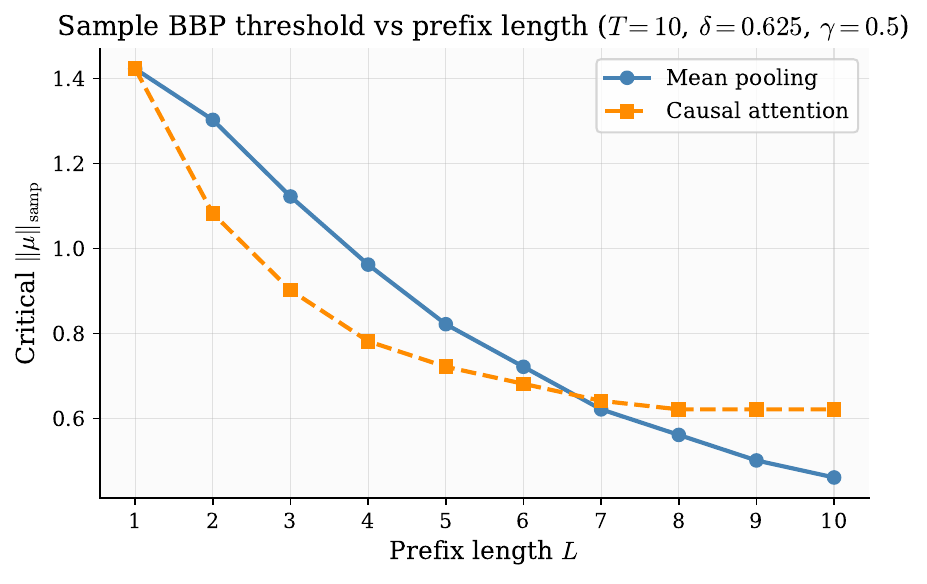}
        \caption{Sample BBP threshold vs.\ prefix length $L$}
        \label{fig:phase_L}
    \end{subfigure}
    \hfill
    \begin{subfigure}[b]{0.48\textwidth}
        \centering
        \includegraphics[width=\textwidth]{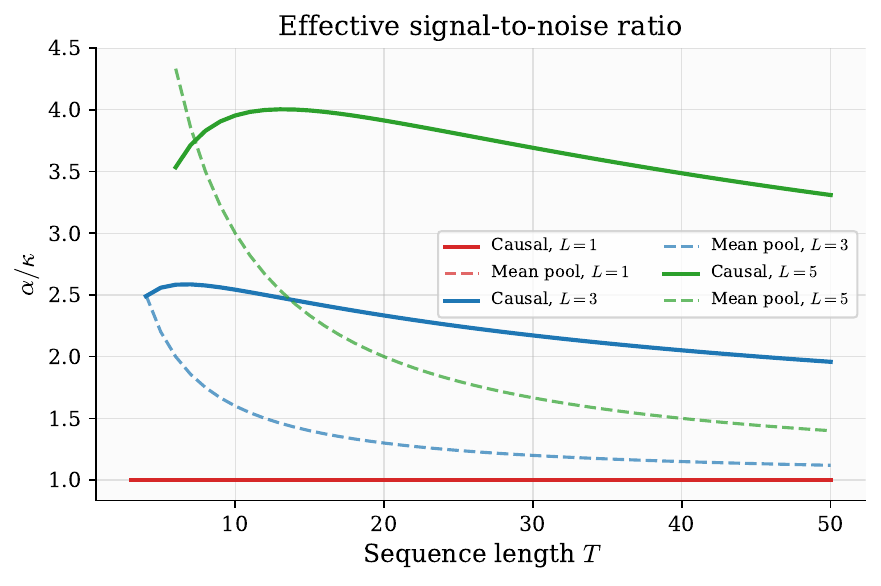}
        \caption{$\alpha/\kappa$ ratio vs.\ sequence length $T$}
        \label{fig:snr}
    \end{subfigure}
    \caption{\textbf{(a)} Critical $\|\vmu\|$ for signal detectability: causal attention has a lower threshold for small $L$, confirming its benefit when signal concentrates early. \textbf{(b)} Effective SNR $\alpha/\kappa$ vs.\ $T$: causal attention (solid) consistently exceeds mean pooling (dashed) for $L\geq 2$. The gap is stable as $T$ grows, while mean pooling's SNR decays as $1/T$.}
    \label{fig:phase_transitions}
\end{figure}

\paragraph{Phase transitions and SNR (Figure~\ref{fig:phase_transitions}).} Figure~\ref{fig:phase_L} shows the critical signal strength $\mu_\mathrm{samp}$ vs.\ prefix length $L$: causal attention has a strictly lower threshold for small $L$, with the gap largest at $L=1$. For large $L$ approaching $T$, the advantage diminishes as most positions become informative. Figure~\ref{fig:snr} shows $\alpha/\kappa$ vs.\ $T$ where causal attention maintains a stable SNR advantage that grows with $T$ for fixed $L$, while mean pooling's SNR decays as more noise positions dilute the signal.

\paragraph{Classification setting.} To illustrate the implications of optimal signal recovery on an ML example, we consider the following binary classification setup. We generate sequences according to the prefix model of Section~\ref{sec:statistical_model} with $d\!=\!300$, $V\!=\!500$, $N\!=\!800$, $T\!=\!10$, $L\!=\!3$, and assign each sequence a binary label $y_n = \mathrm{sign}(\vone_L^\top \xi_n)\in\{\pm 1\}$ with $\vy=(y_1, \ldots, y_N)\in\sR^N$, so that $y_n$ reflects the majority class of the $L$ signal-carrying positions. For each pooling strategy $\vw \in \{\vw^\id, \vw^\cau, \vw^\opt\}$ we build pooled representations $\vc_n = \sum_t w_t X_t^{(n)}\in \sR^d$ and fit a ridge-regularized linear classifier $\hat{\vbeta}_\lambda = (\mC^\top \mC + \lambda N \mI_d)^{-1}\mC^\top \vy$ with $\lambda=1$. We report test accuracy when varying $\|\vmu\|\in[0,5]$. We further consider a variant where $\vw$ is learned along the ridge classifier weights as described in Appendix \ref{appendix:classif}, which can be seen as a simplified setting of modern architectures where attention weights are also trainable parameters.

\begin{wrapfigure}{r}{0.42\textwidth}
    \vspace{-1.0em}
    \centering
    \includegraphics[width=0.40\textwidth]{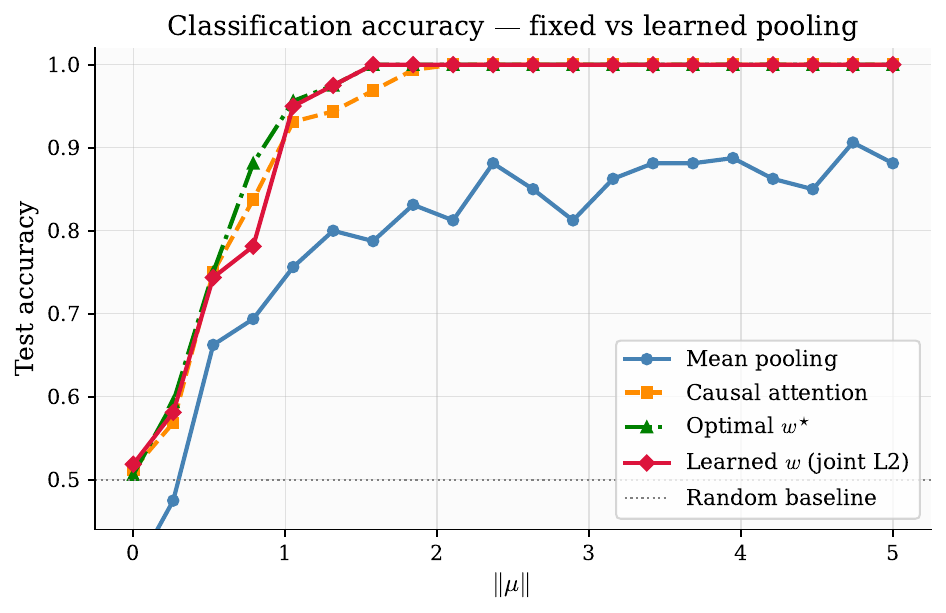}
    \caption{Ridge classification accuracy vs. $\|\vmu\|$ on the prefix model ($d=300$, $V=500$, $N=800$, $T=10$, $L=3$). The learned~$\vw$ baseline jointly minimizes the L2 loss of~\eqref{eq:joint_l2} in Appendix \ref{appendix:classif} over $(\vphi,\vbeta)$, with $\vw=\softmax(\vphi)$ and $\vbeta$ obtained in closed form at every step.}
    \label{fig:classification}
    \vspace{-0.8em}
\end{wrapfigure}

\paragraph{Downstream classification (Figure~\ref{fig:classification}).} Ridge classification accuracy follows the pattern predicted by our spectral analysis. The ordering optimal $>$ causal $>$ mean pooling is maintained across the full range of $\|\vmu\|$, suggesting that the SNR advantage identified by our theory translates qualitatively to downstream task performance, even though the classifier itself is outside the scope of our theorems.
A fully theoretical characterization of the test-accuracy curves in Figure~\ref{fig:classification} can be obtained with random matrix theory by building on the present spectral analysis and adapting the high-dimensional equivalent-resolvent machinery recently developed for Gaussian-mixture classification~\citep{demir2025asymptotic}; we leave this derivation as an interesting direction for future work.

\vspace{0.25em}
Taken together, the experiments above confirm the spectral results developed in Sections~\ref{sec:main}--\ref{sec:causal} and show that the SNR ranking $\lambda_\mathrm{max}(\mR)\ge \alpha_\cau/\kappa_\cau \ge \alpha_\id/\kappa_\id$ predicted by the theory is consistent, quantitatively, across all the considered settings. We close the paper with a summary of these findings and outline several directions in which our framework can naturally be further extended.

\section{Conclusion}\label{sec:conclusion}

We have presented an exact random matrix theory framework for understanding how attention-based pooling affects signal recovery in high-dimensional sequence models. The effective signal strength $\alpha(\vw) = \vw^\top\mR\vw$ and noise scale $\kappa(\vw) = \|\vw\|^2$ jointly determine two BBP-type phase transitions and the resulting eigenvector alignment. We proved that the optimal weights are the normalized top eigenvector of the correlation matrix $\mR$ yielding provably optimal signal recovery. 

\paragraph{Extensions and open problems.} Beyond the setting analyzed here, our framework naturally opens several directions that we believe the present results pave the way for. \textit{(i) Richer embedding models:} Our two-class Gaussian mixture can be extended to $K$-class mixtures. The RMT tools extend to multi-spike models~\citep{benaych2011eigenvalues}, though the algebra becomes more involved. \textit{(ii) Sample-dependent correlations:} In practice, $\mR$ varies across sequences. Extending our framework to $\mR^{(n)}$ that depends on the sample $n$ would require analyzing a mixture of covariance structures, connecting to the theory of mixtures of spiked models. \textit{(iii) Data-dependent attention:} Full self-attention produces weights $\vw^{(n)}$ that depend on the input sequence $\mX^{(n)}$. Analyzing this requires understanding the joint distribution of $(w^{(n)}, \mX^{(n)})$, which is considerably more challenging but could be approached via leave-one-out or cavity methods from statistical physics \citep{mezard2003cavity}. \textit{(iv) Multi-head and multi-layer architectures:} Extending to multiple heads (each with different effective $\alpha/\kappa$) and multiple layers (where the output of one layer feeds into the next) would connect to the theory of products of random matrices and iterated free convolutions. \textit{(v) Non-isotropic noise:} Replacing $\vz_v \sim \gN(\vzero, \mI_d)$ with $\vz_v \sim \gN(\vzero, \mSigma_0)$ for a general covariance $\mSigma_0$, which aligns more with our experiments in Appendix \ref{app:gpt2}, would modify the bulk law and phase transitions, requiring generalized RMT results \citep{silverstein1995empirical}. \textit{(vi) Downstream classification theory:} As discussed in Section~\ref{sec:experiments}, the test-accuracy curves of Figure~\ref{fig:classification} admit an explicit high-dimensional characterization via random matrix theory, building on the present spectral analysis in the spirit of~\citep{firdoussi2024high, demir2025asymptotic}. In each of these directions, the spectral identities derived in this paper can be further generalized and would bring new insights into the understanding of attention-based models.


\bibliographystyle{bib_style}
\bibliography{references}


\input{appendix}

\end{document}

%% file: math_commands.tex
\usepackage{amsmath,amsfonts,bm}

\def\eqref#1{equation~\ref{#1}}

\def\1{\bm{1}}

\def\vzero{{\bm{0}}}
\def\vone{{\bm{1}}}
\def\vmu{{\bm{\mu}}}

\def\vc{{\bm{c}}}

\def\ve{{\bm{e}}}

\def\vs{{\bm{s}}}

\def\vu{{\bm{u}}}
\def\vv{{\bm{v}}}
\def\vw{{\bm{w}}}
\def\vx{{\bm{x}}}
\def\vy{{\bm{y}}}
\def\vz{{\bm{z}}}
\def\vbeta{{\bm{\beta}}}

\def\mA{{\bm{A}}}

\def\mC{{\bm{C}}}

\def\mE{{\bm{E}}}

\def\mG{{\bm{G}}}

\def\mI{{\bm{I}}}

\def\mR{{\bm{R}}}
\def\mS{{\bm{S}}}

\def\mW{{\bm{W}}}
\def\mX{{\bm{X}}}

\def\mZ{{\bm{Z}}}

\def\mSigma{{\bm{\Sigma}}}

\DeclareMathAlphabet{\mathsfit}{\encodingdefault}{\sfdefault}{m}{sl}
\SetMathAlphabet{\mathsfit}{bold}{\encodingdefault}{\sfdefault}{bx}{n}

\def\gE{{\mathcal{E}}}

\def\gK{{\mathcal{K}}}

\def\gN{{\mathcal{N}}}

\def\sR{{\mathbb{R}}}

\newcommand{\E}{\mathbb{E}}

\newcommand{\softmax}{\mathrm{softmax}}



%% file: appendix.tex
\newpage
\appendix

\section{Derivation of the Bulk Cubic Equation (Theorem~\ref{thm:bulk} and Proposition~\ref{prop:edge})}\label{app:cubic}

We derive~\eqref{eq:cubic} using the $S$-transform formalism from free probability~\citep{nica2006lectures}, and then extract the right bulk edge $\lambda_+$ in closed form.

\paragraph{Step 0: Pooled representation and bulk reduction.}
Recall that each pooled row of $\mC$ can be written as
\[
    \vc_n = \sum_{t=1}^T w_t\, X_t^{(n)}
    \;=\; \Big(\sum_{t=1}^T w_t \xi_{n,t}\Big)\vmu + \sum_{t=1}^T w_t\,\vz_{x_{n,t}}.
\]
Using independence of $(\xi_n)$ and $(\vz_v)$ and the correlation structure $\E[\xi_n\xi_n^\top]=\mR$, the (conditional on $\mZ$) population covariance is
\begin{equation}\label{eq:Sigma_pop_app}
    \mSigma \;:=\; \E[\vc_n\vc_n^\top\mid \mZ]
    \;=\; \kappa(\vw)\,\mSigma_\mZ \;+\; \alpha(\vw)\,\|\vmu\|^2\,\vu\vu^\top,
\end{equation}
with $\kappa(\vw)=\|\vw\|^2$, $\alpha(\vw)=\vw^\top\mR\vw$, and $\mSigma_\mZ=(1/V)\mZ\mZ^\top$.
Because $\mSigma$ is a rank-one perturbation of $\kappa\mSigma_\mZ$, the limiting ESD of $\mS=(1/N)\mC\mC^\top$ coincides with that of the null sample covariance
\[
    \mS_0 \;:=\; \tfrac{1}{N}(\kappa\mSigma_\mZ)^{1/2}\mG\mG^\top(\kappa\mSigma_\mZ)^{1/2},
\]
where $\mG\in\sR^{d\times N}$ has i.i.d.\ standard Gaussian entries.

\paragraph{Step 1: Identifying the free multiplicative convolution.}
As $d,V\to\infty$ with $d/V\to\delta$, the ESD of $\mSigma_\mZ$ converges almost surely to the Marchenko--Pastur law $\mathrm{MP}_\delta$, hence the ESD of $\kappa\mSigma_\mZ$ converges to $\kappa\,\mathrm{MP}_\delta$. Likewise the ESD of $(1/N)\mG\mG^\top$ converges to $\mathrm{MP}_\gamma$. By standard free-probability arguments~\citep{nica2006lectures}, the two ensembles $(\kappa\mSigma_\mZ)$ and $(1/N)\mG\mG^\top$ are asymptotically free and the LSD of $\mS_0$ is the free multiplicative convolution
\begin{equation}\label{eq:bulk_free_conv_app}
    \mu_{\mathrm{bulk}} \;=\; \kappa\,\big(\mathrm{MP}_\delta \boxtimes \mathrm{MP}_\gamma\big).
\end{equation}
In particular, the bulk is \emph{not} an MP law unless $\delta=0$ (infinite vocabulary).

\paragraph{Step 2: $S$-transform computation.}
Recall that the $S$-transform of $\mathrm{MP}_a$ is $S_{\mathrm{MP}_a}(z)=1/(1+az)$ and that scaling a distribution by a factor $c>0$ multiplies the $S$-transform by $1/c$. Since $S$-transforms are multiplicative under $\boxtimes$,
\begin{equation}\label{eq:Stransform_F}
    S_F(z) \;=\; \frac{1}{\kappa\,(1+\delta z)(1+\gamma z)},
    \qquad F := \mu_{\mathrm{bulk}}.
\end{equation}

\paragraph{Step 3: From $S$-transform to Stieltjes transform.}
Let $m(z)$ be the Stieltjes transform of $F$ and define the auxiliary function
\[
    \psi(z) \;:=\; z m(z) - 1.
\]
The fundamental identity linking $S_F$ and $m$ is $S_F(\psi(z))\,\psi(z) = (1+\psi(z))/z$, equivalently
\begin{equation}\label{eq:Sfund}
    S_F(\psi(z)) \;=\; \frac{1+\psi(z)}{\psi(z)\cdot z} \;=\; \frac{m(z)}{zm(z)-1}.
\end{equation}
Substituting \eqref{eq:Stransform_F} with $\psi=zm-1$ into \eqref{eq:Sfund} yields
\[
    \frac{1}{\kappa\,(1+\delta(zm-1))(1+\gamma(zm-1))} \;=\; \frac{m}{zm-1}.
\]
Cross-multiplying gives $zm-1=\kappa m(1+\delta(zm-1))(1+\gamma(zm-1))$. Expanding,
\begin{align*}
    \kappa m\Big[1+(\delta+\gamma)(zm-1)+\delta\gamma(zm-1)^2\Big]
    &\;=\; zm-1.
\end{align*}
Collecting powers of $m$ and rearranging yields the cubic equation~\eqref{eq:cubic}:
\begin{equation}\label{eq:cubic_app}
    \delta\gamma\kappa\, z^2 m^3 \;-\; \kappa z(\delta+\gamma-2\delta\gamma)\,m^2 \;-\;\big(z+\kappa(\delta-1)(1-\gamma)\big)\,m \;-\; 1 \;=\; 0.
\end{equation}
The limiting density is recovered by $\rho(x)=\pi^{-1}\lim_{\eta\downarrow0}\Im\,m(x+i\eta)$.

\paragraph{Step 4: Bulk edges and the cubic $R(z)$.}
Let $P(m;z)$ denote the left-hand side of~\eqref{eq:cubic_app}, viewed as a polynomial of degree $3$ in $m$ with coefficients depending on $z$. The edges of the support of $\mu_{\mathrm{bulk}}$ are precisely the real $z$'s where the analytic Stieltjes branch $m(z)$ develops a square-root singularity, i.e.\ where $P(\,\cdot\,;z)$ has a \emph{double root} in $m$:
\[
    P(m;z)=0, \qquad \partial_m P(m;z)=0
    \quad\Longleftrightarrow\quad
    \mathrm{Disc}_m\!\big(P(\cdot;z)\big)=0.
\]
A direct computation (e.g.\ via the resultant of $P$ and $\partial_m P$) shows that the discriminant factorizes as
\begin{equation}\label{eq:disc_factorization}
    \mathrm{Disc}_m\big(P(\cdot;z)\big) \;=\; \kappa\,z^2\,R(z),
\end{equation}
where $R(z)$ is the following cubic in $z$:
\begin{align}
    R(z) \;=\;& 4\delta\gamma\, z^3 \nonumber\\
    &+ \kappa\,\Big(\delta^2\gamma^2 - 10\delta^2\gamma + \delta^2 - 10\delta\gamma^2 - 10\delta\gamma + \gamma^2\Big)\,z^2 \nonumber\\
    &- 2\kappa^2\,\Big(\delta^3\gamma^2 - 4\delta^3\gamma + \delta^3 + \delta^2\gamma^3 + 2\delta^2\gamma^2 + 2\delta^2\gamma + \delta^2 \nonumber\\
    &\qquad\qquad - 4\delta\gamma^3 + 2\delta\gamma^2 - 4\delta\gamma + \gamma^3 + \gamma^2\Big)\, z \nonumber\\
    &+ \kappa^3(\delta-1)^2(\gamma-1)^2(\delta-\gamma)^2.
    \label{eq:Rcubic}
\end{align}
The factor $\kappa z^2$ in~\eqref{eq:disc_factorization} corresponds to the hard edge at $z=0$, while the genuine (right) bulk edge is the largest real root of $R$.

\paragraph{Step 5: Scaling out $\kappa$ and closed-form edge.}
Because the measure is rescaled by $\kappa$, the edges scale linearly in $\kappa$. Writing $z=\kappa x$, \eqref{eq:Rcubic} becomes $R(z)=\kappa^3 r(x)$ with
\begin{align*}
    r(x) \;=\;& 4\delta\gamma\, x^3 \\
    &+ \Big(\delta^2\gamma^2 - 10\delta^2\gamma + \delta^2 - 10\delta\gamma^2 - 10\delta\gamma + \gamma^2\Big)\, x^2 \\
    &- 2\Big(\delta^3\gamma^2 - 4\delta^3\gamma + \delta^3 + \delta^2\gamma^3 + 2\delta^2\gamma^2 + 2\delta^2\gamma + \delta^2 \\
    &\qquad\qquad - 4\delta\gamma^3 + 2\delta\gamma^2 - 4\delta\gamma + \gamma^3 + \gamma^2\Big)\, x \\
    &+ (\delta-1)^2(\gamma-1)^2(\delta-\gamma)^2.
\end{align*}
Introducing the symmetric combinations $q:=\delta\gamma$ and $r_0:=\delta+\gamma+\delta\gamma$, $r(x)$ takes the form $Ax^3+Bx^2+Cx+D$ with $A=4q$ and $B=r_0^2-12 q r_0+12 q^2-12 q$. Applying the standard trigonometric cubic formula with invariants
\[
    \Delta_0 \;=\; B^2-3AC \;=\; r_0\big(r_0^{\,3}+216 q^2\big),
    \qquad
    \Delta_1 \;=\; 2B^3-9ABC+27A^2 D \;=\; 2\big(r_0^{\,6}-540 q^2 r_0^{\,3}-5832 q^4\big),
\]
and letting $\phi=\arccos\!\big(\Delta_1/(2\Delta_0^{3/2})\big)$, the three real roots are
\[
    x_k \;=\; -\frac{1}{3A}\bigg(B + 2\sqrt{\Delta_0}\,\cos\!\Big(\tfrac{\phi+2\pi k}{3}\Big)\bigg),\qquad k=0,1,2.
\]
The right bulk edge of $\mu_{\mathrm{bulk}}$ is
\[
    \lambda_+ \;=\; \kappa\,x_+, \qquad x_+ := \max\{x_0,x_1,x_2\},
\]
which is exactly the expression stated in Proposition~\ref{prop:edge}.

\section{Population Outlier and Overlap (Theorem~\ref{thm:pop_bbp})}\label{app:pop}

Throughout this section let $H$ denote the LSD of $\kappa\mSigma_\mZ$, namely $H=\kappa\,\mathrm{MP}_\delta$, and $m_H$ its Stieltjes transform. The population covariance~\eqref{eq:Sigma_pop_app} is the rank-one additive deformation
\[
    \mSigma \;=\; \kappa\mSigma_\mZ \;+\; \alpha\|\vmu\|^2\,\vu\vu^\top.
\]

\paragraph{Outlier equation.}
By~\citet{benaych2011eigenvalues}, if an outlier eigenvalue $\beta$ of $\mSigma$ lies outside the support of $H$, it must satisfy
\begin{equation}\label{eq:pop_BGN}
    1 \;+\; \alpha\|\vmu\|^2\, m_H(\beta) \;=\; 0
    \quad\Longleftrightarrow\quad
    m_H(\beta) \;=\; -\frac{1}{\alpha\|\vmu\|^2}.
\end{equation}
The rescaling identity $m_H(\beta)=\kappa^{-1} m_{\mathrm{MP}_\delta}(\beta/\kappa)$ turns this into $m_{\mathrm{MP}_\delta}(\beta/\kappa) = -\kappa/(\alpha\|\vmu\|^2) = -1/\rho$, where $\rho := \alpha\|\vmu\|^2/\kappa$ is the population spike ratio.

\paragraph{Solving the MP quadratic.}
The Stieltjes transform of $\mathrm{MP}_\delta$ satisfies
\begin{equation}\label{eq:MP_quadratic}
    \delta\, z\, m^2 \;+\; (z+\delta-1)\,m \;+\; 1 \;=\; 0.
\end{equation}
Setting $m=-1/\rho$ and multiplying through by $\rho^2$:
\[
    \delta z \;-\; (z+\delta-1)\rho \;+\; \rho^2 \;=\; 0
    \quad\Longrightarrow\quad
    z \;=\; \frac{\rho(\rho-\delta+1)}{\rho-\delta}.
\]
Multiplying by $\kappa$ yields the population outlier location
\begin{equation}\label{eq:beta_out_app}
    \beta_\mathrm{out} \;=\; \kappa\,\frac{\rho(\rho-\delta+1)}{\rho-\delta}.
\end{equation}

\paragraph{BBP condition.}
For~\eqref{eq:beta_out_app} to be an actual outlier of $H=\kappa\mathrm{MP}_\delta$, we need $\beta_\mathrm{out}/\kappa>(1+\sqrt\delta)^2$, i.e.
\[
    \frac{\rho(\rho-\delta+1)}{\rho-\delta} \;>\; (1+\sqrt\delta)^2.
\]
Clearing the denominator (for $\rho>\delta$) and expanding gives $\rho^2 - (\delta+2\sqrt\delta)\rho - (\sqrt\delta)(\delta+\sqrt\delta) > 0$, which factors as $(\rho-(\delta+\sqrt\delta))(\rho+\sqrt\delta)>0$. Since $\rho>0$ the BBP condition reduces to
\[
    \rho \;>\; \delta+\sqrt\delta.
\]

\paragraph{Eigenvector overlap.}
The general formula of~\citet{benaych2011eigenvalues} reads
\begin{equation}\label{eq:BGN_overlap}
    |\hat\vu_\mSigma^\top\vu|^2 \;=\; \frac{1}{\alpha^2\|\vmu\|^4\, m_H'(\beta_\mathrm{out})}.
\end{equation}
We compute $m_H'$ by differentiating \eqref{eq:MP_quadratic} implicitly. Writing $G(z,m):=\delta z m^2+(z+\delta-1)m+1=0$,
\[
    m'(z) \;=\; -\frac{\partial_z G}{\partial_m G}
    \;=\; -\frac{\delta m^2 + m}{2\delta z m + (z+\delta-1)}.
\]
At $(z_\star,m_\star)=(\beta_\mathrm{out}/\kappa, -1/\rho)$, the numerator is $\delta/\rho^2 - 1/\rho = (\delta-\rho)/\rho^2$, and using~\eqref{eq:MP_quadratic} at this point to express the denominator: from $G(z_\star,m_\star)=0$ we get $\delta z_\star/\rho^2 - (z_\star+\delta-1)/\rho + 1 = 0$, which after algebraic manipulation gives $2\delta z_\star m_\star + (z_\star+\delta-1) = -\delta z_\star/\rho - 1/m_\star = \rho - \delta z_\star/\rho$. Using $z_\star=\rho(\rho-\delta+1)/(\rho-\delta)$, a direct computation yields
\[
    m_{\mathrm{MP}_\delta}'(z_\star) \;=\; \frac{1}{\rho^2}\cdot \frac{(\rho-\delta)^2}{(\rho-\delta)^2-\delta}.
\]
Translating back through $m_H'(\beta_\mathrm{out}) = \kappa^{-2} m_{\mathrm{MP}_\delta}'(\beta_\mathrm{out}/\kappa)$ and substituting in~\eqref{eq:BGN_overlap} with $\alpha^2\|\vmu\|^4=\kappa^2\rho^2$:
\begin{equation}\label{eq:overlap_pop_app}
    |\hat\vu_\mSigma^\top\vu|^2 \;=\; 1 - \frac{\delta}{(\rho-\delta)^2}.
\end{equation}
At the threshold $\rho=\delta+\sqrt\delta$, $(\rho-\delta)^2=\delta$ and~\eqref{eq:overlap_pop_app} equals zero, confirming the phase transition.

\section{Sample Outlier Quadratic (Theorem~\ref{thm:sample_spike})}\label{app:sample_spike}

Assume we are in the supercritical regime, so that $\mSigma$ has an outlier eigenvalue $\beta:=\beta_\mathrm{out}$ separated from the bulk. Viewing $\mSigma$ as a rank-one deformation of its bulk component, the standard theory of spiked sample covariance matrices (\citealp{silverstein1995empirical,bai2010spectral,benaych2011eigenvalues}) states that an outlier $\lambda_\mathrm{out}$ of $\mS=(1/N)\mC\mC^\top$ is characterized by the \emph{companion outlier equation}
\begin{equation}\label{eq:outlier_eq_app}
    \underline{m}(\lambda_\mathrm{out}) \;=\; -\frac{1}{\beta},
\end{equation}
where $\underline{m}(z) = -(1-\gamma)/z + \gamma m(z)$ is the companion Stieltjes transform (defined after Theorem~\ref{thm:bulk}).

\paragraph{Step 1: Linearizing $m$ at the outlier.}
Equation~\eqref{eq:outlier_eq_app} gives a linear relation between $m$ and $\lambda$ at the outlier:
\begin{equation}\label{eq:m_out_app}
    m_\mathrm{out} \;:=\; m(\lambda_\mathrm{out}) \;=\; \frac{1-\gamma}{\gamma\,\lambda_\mathrm{out}} \;-\; \frac{1}{\gamma\beta}.
\end{equation}

\paragraph{Step 2: Substitution into the cubic.}
Plugging $m=m_\mathrm{out}$ and $z=\lambda:=\lambda_\mathrm{out}$ into the bulk cubic~\eqref{eq:cubic_app}:
\[
    \delta\gamma\kappa\,\lambda^2 m_\mathrm{out}^3 \;-\; \kappa\lambda(\delta+\gamma-2\delta\gamma)\,m_\mathrm{out}^2 \;-\; (\lambda+\kappa(\delta-1)(1-\gamma))\,m_\mathrm{out} \;-\; 1 \;=\; 0.
\]
Let us denote $a:= (1-\gamma)/\gamma$ and $b := 1/(\gamma\beta)$, so $m_\mathrm{out} = a/\lambda - b$. Multiplying the equation above by $\lambda^2$ clears $m_\mathrm{out}$ denominators, yielding a polynomial equation in $\lambda$. Expanding and collecting:
\begin{itemize}
    \item The cubic term gives $\delta\gamma\kappa\,(a-b\lambda)^3$, which contributes $-\delta\gamma\kappa b^3\lambda^3 + 3\delta\gamma\kappa a b^2\lambda^2 - 3\delta\gamma\kappa a^2 b\lambda + \delta\gamma\kappa a^3$.
    \item The quadratic term gives $-\kappa\lambda(\delta+\gamma-2\delta\gamma)(a-b\lambda)^2/\lambda^0 = -\kappa(\delta+\gamma-2\delta\gamma)\lambda(a-b\lambda)^2$.
    \item The linear term gives $-(\lambda+\kappa(\delta-1)(1-\gamma))\cdot\lambda(a-b\lambda)$.
    \item The constant term $-\lambda^2$.
\end{itemize}
Combining the above, an overall factor of $\lambda$ appears in the resulting polynomial expression (corresponding to the trivial root $\lambda=0$). After dividing by $\lambda$ and regrouping, the remaining quadratic in $\lambda$ is exactly~\eqref{eq:lambda_quad}:
\begin{equation}\label{eq:lambda_quad_app}
    \delta\kappa\,\lambda^2 \;+\; \big(-\beta^2\gamma + \beta\kappa(\gamma(1+\delta) - 2\delta)\big)\,\lambda \;+\; \beta^2\big(\beta\gamma + \kappa(1-\gamma)(\delta-\gamma)\big) \;=\; 0.
\end{equation}
Writing $A_\lambda=\delta\kappa$, $B_\lambda=-\beta^2\gamma+\beta\kappa(\gamma(1+\delta)-2\delta)$, $C_\lambda=\beta^2(\beta\gamma+\kappa(1-\gamma)(\delta-\gamma))$, the two roots are
\[
    \lambda_{\pm} \;=\; \frac{-B_\lambda\pm\sqrt{B_\lambda^2-4A_\lambda C_\lambda}}{2A_\lambda}.
\]

\paragraph{Step 3: Selecting the physical root.}
The sample outlier equals the larger root, $\lambda_\mathrm{out}=\lambda_+$, whenever $\lambda_+$ exceeds the bulk edge $\lambda_+^{\mathrm{bulk}}$ of Appendix~\ref{app:cubic}. Equivalently, $\beta$ must exceed the sample BBP threshold $\beta_\mathrm{crit} = -1/\underline{m}(\lambda_+^{\mathrm{bulk},+})$. Below $\beta_\mathrm{crit}$, equation~\eqref{eq:outlier_eq_app} has no real solution in the supercritical branch, and no outlier separates from the bulk.

\section{Sample Eigenvector Overlap (Theorem~\ref{thm:sample_overlap})}\label{app:sample_overlap}

This section contains the main ingredient: a self-contained proof of the key identity
\begin{equation}\label{eq:overlap_sample_goal}
    |\hat\vu_\mS^\top\,\hat\vu_\mSigma|^2 \;=\; \frac{1}{\beta\,\lambda_\mathrm{out}\,\underline{m}'(\lambda_\mathrm{out})}.
\end{equation}
The crucial $1/\lambda_\mathrm{out}$ factor (absent in additive/Wigner deformations) arises from differentiating the \emph{multiplicative} structure of the sample covariance.

\subsection{A differential identity: overlap as \texorpdfstring{$(\beta/\lambda)\,d\lambda/d\beta$}{(beta/lambda) dlambda/dbeta}}

Fix a realization of the Wishart noise
\[
    \mW \;:=\; \frac{1}{N}\mG^\top\mG \in \sR^{d\times d},
    \qquad G_{ij}\stackrel{\mathrm{i.i.d.}}{\sim}\gN(0,1),
\]
and consider the sample covariance
\[
    \mS(\beta) \;=\; \mSigma(\beta)^{1/2}\mW\,\mSigma(\beta)^{1/2},
\]
where we parametrize the population covariance so that its top eigenpair is frozen:
\begin{equation}\label{eq:Sigma_beta_decomp}
    \mSigma(\beta) \;=\; \beta\,\hat\vu_\mSigma\hat\vu_\mSigma^\top \;+\; \mSigma_\perp,
    \qquad \hat\vu_\mSigma^\top\mSigma_\perp=\vzero, \qquad \|\hat\vu_\mSigma\|=1.
\end{equation}
Here $\mSigma_\perp$ gathers the remaining eigen-data and does \emph{not} depend on $\beta$. Let $\lambda(\beta)$ be a simple outlier eigenvalue of $\mS(\beta)$ with associated unit eigenvector $\hat\vu_\mS(\beta)$.

\begin{lemma}[Overlap/derivative identity]\label{lem:overlap_derivative}
    Under~\eqref{eq:Sigma_beta_decomp},
    \begin{equation}\label{eq:overlap_derivative}
        |\hat\vu_\mS(\beta)^\top\,\hat\vu_\mSigma|^2
        \;=\; \frac{\beta}{\lambda(\beta)}\,\frac{d\lambda(\beta)}{d\beta}.
    \end{equation}
\end{lemma}
\begin{proof}
In the eigenbasis of~\eqref{eq:Sigma_beta_decomp}, $\mSigma(\beta)^{1/2}=\sqrt\beta\,\hat\vu_\mSigma\hat\vu_\mSigma^\top + \mSigma_\perp^{1/2}$ and so
\[
    \frac{d}{d\beta}\mSigma(\beta)^{1/2} \;=\; \frac{1}{2\sqrt\beta}\,\hat\vu_\mSigma\hat\vu_\mSigma^\top.
\]
Differentiating $\mS(\beta)=\mSigma^{1/2}\mW\mSigma^{1/2}$,
\[
    \frac{d\mS}{d\beta} \;=\; \frac{d\mSigma^{1/2}}{d\beta}\mW\mSigma^{1/2} + \mSigma^{1/2}\mW\frac{d\mSigma^{1/2}}{d\beta}.
\]
For a simple eigenpair, standard first-order perturbation gives $d\lambda/d\beta = \hat\vu_\mS^\top (d\mS/d\beta)\hat\vu_\mS$. Substituting and using symmetry of $\mW$:
\[
    \frac{d\lambda}{d\beta} \;=\; \frac{1}{\sqrt\beta}\,\big(\hat\vu_\mSigma^\top\hat\vu_\mS\big)\,\hat\vu_\mSigma^\top\mW\mSigma^{1/2}\hat\vu_\mS.
\]
Now left-multiply the eigenequation $\mS\hat\vu_\mS=\lambda\hat\vu_\mS$ by $\hat\vu_\mSigma^\top$ and use $\hat\vu_\mSigma^\top\mSigma^{1/2}=\sqrt\beta\,\hat\vu_\mSigma^\top$:
\[
    \sqrt\beta\,\hat\vu_\mSigma^\top\mW\mSigma^{1/2}\hat\vu_\mS \;=\; \lambda\,\hat\vu_\mSigma^\top\hat\vu_\mS,
    \qquad\text{i.e.}\qquad
    \hat\vu_\mSigma^\top\mW\mSigma^{1/2}\hat\vu_\mS \;=\; \frac{\lambda}{\sqrt\beta}\,\hat\vu_\mSigma^\top\hat\vu_\mS.
\]
Plugging back,
\[
    \frac{d\lambda}{d\beta} \;=\; \frac{\lambda}{\beta}\,\big(\hat\vu_\mSigma^\top\hat\vu_\mS\big)^2,
\]
which rearranges to~\eqref{eq:overlap_derivative}.
\end{proof}

\subsection{Differentiating the outlier equation}

In the high-dimensional limit, the outlier $\lambda_\mathrm{out}$ produced by population spike $\beta$ is characterized by~\eqref{eq:outlier_eq_app}. Implicit differentiation of $\underline{m}(\lambda_\mathrm{out}(\beta))=-1/\beta$ with respect to $\beta$ gives
\begin{equation}\label{eq:dlambda_dbeta}
    \underline{m}'(\lambda_\mathrm{out})\,\frac{d\lambda_\mathrm{out}}{d\beta} \;=\; \frac{1}{\beta^2}
    \;\Longrightarrow\;
    \frac{d\lambda_\mathrm{out}}{d\beta} \;=\; \frac{1}{\beta^2\,\underline{m}'(\lambda_\mathrm{out})}.
\end{equation}
Substituting~\eqref{eq:dlambda_dbeta} into Lemma~\ref{lem:overlap_derivative} yields~\eqref{eq:overlap_sample_goal}.

\subsection{Explicit computation of \texorpdfstring{$\underline{m}'(\lambda_\mathrm{out})$}{m'(lambda\_out)}}

From $\underline{m}(z)=-(1-\gamma)/z+\gamma m(z)$ we obtain
\begin{equation}\label{eq:umprime_app}
    \underline{m}'(z) \;=\; \frac{1-\gamma}{z^2} + \gamma\,m'(z).
\end{equation}
To compute $m'(\lambda_\mathrm{out})$, write the bulk cubic~\eqref{eq:cubic_app} as $F(\lambda,m)=0$. Implicit differentiation gives $m'(\lambda) = -\partial_\lambda F(\lambda,m)/\partial_m F(\lambda,m)$ with
\begin{align}
    \partial_m F(\lambda,m) &\;=\; 3\delta\gamma\kappa\,\lambda^2 m^2 \;-\; 2\kappa\lambda(\delta+\gamma-2\delta\gamma)\,m \;-\; \big(\lambda+\kappa(\delta-1)(1-\gamma)\big), \label{eq:Fm_app}\\
    \partial_\lambda F(\lambda,m) &\;=\; 2\delta\gamma\kappa\,\lambda m^3 \;-\; \kappa(\delta+\gamma-2\delta\gamma)\,m^2 \;-\; m. \label{eq:Fl_app}
\end{align}
Evaluating at $(\lambda_\mathrm{out},m_\mathrm{out})$ with $m_\mathrm{out}$ given by~\eqref{eq:m_out_app} produces explicit values for $m'(\lambda_\mathrm{out})$ and, via~\eqref{eq:umprime_app}, for $\underline{m}'(\lambda_\mathrm{out})$. Substituting the latter into~\eqref{eq:overlap_sample_goal} yields $|\hat\vu_\mS^\top\hat\vu_\mSigma|^2$ in closed form, completing the proof of Theorem~\ref{thm:sample_overlap}.

\subsection{Consistency check: \texorpdfstring{$\delta\to 0$}{delta -> 0} reduces to Paul's formula}

When $\delta\to 0$ (infinite vocabulary), $\mSigma_\mZ\to \mI_d$ and the bulk reduces to the scaled MP law $\kappa\,\mathrm{MP}_\gamma$. In this limit, the companion transform admits a closed form and the outlier location simplifies to $\lambda_\mathrm{out}=\beta(1+\gamma/(\beta-1))$. A direct substitution into~\eqref{eq:overlap_sample_goal} recovers the classical formula of~\citet{paul2007asymptotics}:
\[
    |\hat\vu_\mS^\top\hat\vu_\mSigma|^2 \;=\; \frac{1-\gamma/(\beta-1)^2}{1+\gamma/(\beta-1)}.
\]

\section{Phase Transition Thresholds (Section~\ref{sec:phase})}\label{app:phase}

Recall the population spike ratio $\rho=\alpha(\vw)\|\vmu\|^2/\kappa(\vw)$ and the population outlier location~\eqref{eq:beta_out_app}.

\paragraph{Population BBP threshold.}
From Appendix~\ref{app:pop}, an outlier appears iff $\rho>\delta+\sqrt\delta$. Expressing this in terms of $\|\vmu\|$,
\begin{equation}\label{eq:mu_pop_app}
    \|\vmu\| \;>\; \mu_\mathrm{pop}(\vw) \;:=\; \sqrt{\frac{\kappa(\vw)(\delta+\sqrt\delta)}{\alpha(\vw)}}.
\end{equation}

\paragraph{Sample BBP threshold.}
Even when $\mSigma$ has an outlier $\beta_\mathrm{out}$, the sample covariance exhibits a separated outlier only if $\beta_\mathrm{out}$ exceeds the sample critical value
\begin{equation}\label{eq:beta_crit_app}
    \beta_\mathrm{crit} \;:=\; -\frac{1}{\underline{m}(\lambda_+^{+})},
\end{equation}
where $\lambda_+$ is the right edge of the bulk LSD~\eqref{eq:bulk_free_conv_app}. The critical ratio $\rho_\mathrm{samp}$ satisfying $\beta_\mathrm{out}(\rho_\mathrm{samp})=\beta_\mathrm{crit}$ is obtained by multiplying~\eqref{eq:beta_out_app} by $(\rho-\delta)$ and rearranging:
\[
    \kappa\rho(\rho-\delta+1) \;=\; \beta_\mathrm{crit}(\rho-\delta)
    \;\Longleftrightarrow\;
    \kappa\rho^2 + \big(\kappa(1-\delta)-\beta_\mathrm{crit}\big)\rho + \beta_\mathrm{crit}\delta \;=\; 0.
\]
Among the two roots, $\rho_\mathrm{samp}$ is the one with $\rho_\mathrm{samp}>\delta+\sqrt\delta$. The corresponding $\|\vmu\|$-threshold is
\begin{equation}\label{eq:mu_samp_app}
    \mu_\mathrm{samp}(\vw) \;=\; \sqrt{\frac{\kappa(\vw)\,\rho_\mathrm{samp}}{\alpha(\vw)}}.
\end{equation}
Both $\mu_\mathrm{pop}$ and $\mu_\mathrm{samp}$ are decreasing functions of the effective SNR ratio $\alpha(\vw)/\kappa(\vw)$, which confirms that maximizing $\alpha/\kappa$ uniformly lowers both detectability thresholds.

\section{Optimal Attention Weights (Theorem~\ref{thm:optimal})}\label{app:optimal}

\paragraph{Setup.}
Since $\mR$ is a correlation matrix, it is symmetric (and positive semi-definite with unit diagonal). The objective is the Rayleigh quotient
\[
    f(\vw) \;:=\; \frac{\vw^\top \mR\vw}{\vw^\top\vw},
\]
subject to the affine constraint $\vone^\top\vw = 1$.

\paragraph{Scale invariance.}
For any $\vx \neq \vzero$ with $\vone^\top\vx \neq 0$, set $\vw = \vx/(\vone^\top\vx)$. Then $\vone^\top\vw = 1$ and
\[
    f(\vw) \;=\; \frac{\vx^\top\mR\vx}{\vx^\top\vx} \;=\; f(\vx).
\]
So the constrained problem $\max\{f(\vw) : \vone^\top\vw = 1\}$ is equivalent to maximizing the unconstrained Rayleigh quotient over directions $\vx$ with $\vone^\top\vx\neq 0$.

\paragraph{Rayleigh--Ritz bound.}
By the Rayleigh--Ritz theorem, $f(\vx) \leq \lambda_{\mathrm{max}}(\mR)$ for all $\vx\neq\vzero$, with equality iff $\vx \in \gE_{\mathrm{max}}$, the eigenspace associated to $\lambda_{\mathrm{max}}(\mR)$. Therefore $\sup_{\vone^\top\vw=1}f(\vw) = \lambda_{\mathrm{max}}(\mR)$.

\paragraph{Attainability.}
A maximizer $\vw^\star$ exists iff there exists $\vv \in \gE_{\mathrm{max}}$ with $\vone^\top\vv\neq 0$, in which case $\vw^\star = \vv/(\vone^\top\vv)$ achieves $f(\vw^\star) = \lambda_{\mathrm{max}}(\mR)$. If $\vone \perp \gE_{\mathrm{max}}$, the supremum is not attained (a maximizing sequence $\vw_k$ has $\|\vw_k\|\to\infty$). When $\lambda_{\mathrm{max}}(\mR)$ is simple, $\gE_{\mathrm{max}} = \mathrm{span}\{\vv_{\mathrm{max}}\}$, so the optimizer is unique up to the sign fixed by normalization.

\section{Deterministic Weight Limit (Lemma~\ref{lem:det_weights})}\label{app:det_weights}

Consider parameter-free causal self-attention with $Q=K=V=\mX$, raw scores $S_{t,s}^\mathrm{raw} = \frac{\tau}{d}\langle X_t, X_s\rangle$, causal mask ($S_{t,s}\leftarrow -\infty$ for $s>t$), no-self mask ($S_{t,t}\leftarrow -\infty$, except the degenerate first row, where $S_{1,1}:=0$). Let $\gK_t$ be the admissible key set at row $t$: $\gK_1=\{1\}$, $\gK_t=\{1,\ldots,t-1\}$ for $t\geq 2$, $k_t:=|\gK_t|$.

\paragraph{Step 1: Admissible scores are $O_P(d^{-1/2})$.}
Under~\eqref{eq:embedding}, for distinct positions $t\neq s$:
\[
    \langle X_t, X_s\rangle \;=\; s_{x_t}s_{x_s}\|\vmu\|^2 + s_{x_t}\vmu^\top\vz_{x_s} + s_{x_s}\vz_{x_t}^\top\vmu + \vz_{x_t}^\top\vz_{x_s}.
\]
With $\vz_v\sim\gN(\vzero,\mI_d)$ i.i.d.\ and $\|\vmu\|=O(1)$, the four terms are $O(1)$, $O_P(\sqrt d)$, $O_P(\sqrt d)$, $O_P(\sqrt d)$ respectively. The dominant term is $\vz_{x_t}^\top\vz_{x_s}=O_P(\sqrt d)$. Therefore, under the $\tau/d$ scaling, $S_{t,s}^\mathrm{raw}=O_P(d^{-1/2})$ uniformly over admissible $(t,s)$ pairs since $T=O(1)$.

\paragraph{Step 2: Row-wise Taylor expansion of the softmax.}
Fix row $t$ and let $\vs^{(t)}\in\sR^{k_t}$ denote the vector of admissible scores. For $\|\vs^{(t)}\|_\infty$ small,
\begin{equation}\label{eq:softmax_taylor_app}
    \softmax(\vs^{(t)}) \;=\; \frac{1}{k_t}\vone \;+\; \frac{1}{k_t}\Big(\vs^{(t)} - \bar{s}_t\,\vone\Big) \;+\; O\!\big(\|\vs^{(t)}\|_\infty^{\,2}\big),
    \qquad \bar s_t \;:=\; \frac{1}{k_t}\vone^\top\vs^{(t)}.
\end{equation}
This expansion is obtained by writing $\softmax(\vs)_i=e^{s_i}/\sum_j e^{s_j}$ and Taylor-expanding at $\vs=\vzero$.

\paragraph{Step 3: Row-then-column averaging preserves the $O_P(d^{-1/2})$ error.}
From~\eqref{eq:softmax_taylor_app}, each row of $\mA$ differs from its uniform counterpart by $O_P(d^{-1/2})$ in $\ell_\infty$, hence in $\ell_1$ since $k_t\le T=O(1)$. The pooled attention weight vector is $\vw^\mathrm{att}=T^{-1}\mA^\top\vone$, a finite convex combination of rows. Averaging preserves the $O_P(d^{-1/2})$ error (because $T$ is fixed), so
\[
    \vw^\mathrm{att} \;=\; \vw^{(0)} + O_P(d^{-1/2}),
\]
where $\vw^{(0)}$ is obtained by replacing each admissible row by its uniform distribution on $\gK_t$.

\section{Harmonic Weights and Causal Scalars (Proposition~\ref{prop:harmonic})}\label{app:harmonic}

\paragraph{Proof of Proposition~\ref{prop:harmonic}.}
In the score-free limit of Appendix~\ref{app:det_weights}, the causal attention matrix is
\[
    A_{1,1}^{(0)}=1,\qquad A_{t,s}^{(0)}=\frac{1}{t-1}\ \text{ for } t\geq 2,\ s\leq t-1,\qquad A_{t,s}^{(0)}=0\ \text{otherwise}.
\]
The pooled weight at position $s$ is $w_s^{(0)}=T^{-1}\sum_{t=1}^T A_{t,s}^{(0)}$. For $s=1$, every row $t\geq 2$ contributes $1/(t-1)$ to key $1$ and row $t=1$ contributes $1$:
\[
    w_1^{(0)} \;=\; \frac{1}{T}\Big(1+\sum_{t=2}^T\frac{1}{t-1}\Big) \;=\; \frac{1+\mathcal{H}_{T-1}}{T}.
\]
For $s\geq 2$, key $s$ only appears for rows $t\geq s+1$ (it is masked out of earlier rows):
\[
    w_s^{(0)} \;=\; \frac{1}{T}\sum_{t=s+1}^T \frac{1}{t-1} \;=\; \frac{1}{T}\sum_{k=s}^{T-1}\frac{1}{k} \;=\; \frac{\mathcal{H}_{T-1}-\mathcal{H}_{s-1}}{T}.
\]

\paragraph{Harmonic identities.}
We record two classical identities, used repeatedly below:
\begin{align}
    \sum_{k=1}^n \mathcal{H}_k &\;=\; (n+1)\mathcal{H}_n - n, \label{eq:harm_id1}\\
    \sum_{k=1}^n \mathcal{H}_k^2 &\;=\; (n+1)\mathcal{H}_n^2 - (2n+1)\mathcal{H}_n + 2n. \label{eq:harm_id2}
\end{align}
Both are proved by writing $\mathcal{H}_k=\sum_{j\le k}1/j$ and exchanging the order of summation.

\paragraph{Closed form for $\kappa_\cau=\|\vw^\cau\|^2$.}
From $\vw^\cau=\vw^{(0)}$:
\begin{align*}
    T^2\kappa_\cau
    &= (1+\mathcal{H}_{T-1})^2 + \sum_{s=2}^T(\mathcal{H}_{T-1}-\mathcal{H}_{s-1})^2 \\
    &= (1+\mathcal{H}_{T-1})^2 + \sum_{j=0}^{T-2}(\mathcal{H}_{T-1}-\mathcal{H}_j)^2 \qquad (j=s-1)\\
    &= (1+\mathcal{H}_{T-1})^2 + (T-1)\mathcal{H}_{T-1}^2 - 2\mathcal{H}_{T-1}\sum_{j=0}^{T-2}\mathcal{H}_j + \sum_{j=0}^{T-2}\mathcal{H}_j^2.
\end{align*}
Applying \eqref{eq:harm_id1}--\eqref{eq:harm_id2} (and $\mathcal{H}_0=0$) to $\sum_{j=1}^{T-2}\mathcal{H}_j$ and $\sum_{j=1}^{T-2}\mathcal{H}_j^2$, we obtain after straightforward algebraic collapse
\begin{equation}\label{eq:kappa_cau_app}
    \kappa_\cau \;=\; \frac{2T-1+\mathcal{H}_{T-1}}{T^2}.
\end{equation}

\paragraph{Closed forms for the prefix mass and suffix square sum.}
For the prefix model ($\mR=\vone_L\vone_L^\top\oplus\mI_{T-L}$), define $\Delta_{L,T}:=\mathcal{H}_{T-1}-\mathcal{H}_{L-1}$. Then
\begin{align}
    \sum_{t=1}^L w_t^\cau
    &\;=\; \frac{1+\mathcal{H}_{T-1}}{T} + \frac{1}{T}\sum_{t=2}^L\big(\mathcal{H}_{T-1}-\mathcal{H}_{t-1}\big) \nonumber\\
    &\;=\; \frac{1+L\mathcal{H}_{T-1}-\sum_{t=1}^{L-1}\mathcal{H}_t}{T}
    \;=\; \frac{L\big(1+\Delta_{L,T}\big)}{T}, \label{eq:prefix_mass_app}
\end{align}
where the second equality uses \eqref{eq:harm_id1} in the form $\sum_{t=1}^{L-1}\mathcal{H}_t=L\mathcal{H}_{L-1}-(L-1)$. For the suffix square sum, expanding $(\mathcal{H}_{T-1}-\mathcal{H}_{s-1})^2$ and summing over $s=L+1,\ldots,T$:
\begin{align}
    T^2\sum_{t=L+1}^T(w_t^\cau)^2
    &\;=\; (T-L)\mathcal{H}_{T-1}^2 - 2\mathcal{H}_{T-1}\!\!\sum_{j=L}^{T-1}\mathcal{H}_j + \sum_{j=L}^{T-1}\mathcal{H}_j^2 \nonumber\\
    &\;=\; 2(T-L) + (1-2L)\Delta_{L,T} - L\Delta_{L,T}^2. \label{eq:suffix_sq_app}
\end{align}
Combining $\alpha_\cau = (\sum_{t=1}^L w_t^\cau)^2 + \sum_{t=L+1}^T (w_t^\cau)^2$ with \eqref{eq:prefix_mass_app}--\eqref{eq:suffix_sq_app} produces the closed form \eqref{eq:alpha_cau_full}:
\begin{equation}\label{eq:alpha_cau_full}
    \alpha_\cau \;=\; \frac{1}{T^2}\Big(L^2 + 2(T-L) + (2L(L-1)+1)\Delta_{L,T} + L(L-1)\Delta_{L,T}^2\Big).
\end{equation}

\paragraph{Comparison with mean pooling.}
For $\vw^\id=\vone/T$, $\kappa_\id=1/T$ and $\alpha_\id=(L^2+(T-L))/T^2$. A direct inspection of \eqref{eq:kappa_cau_app}--\eqref{eq:alpha_cau_full} shows that $\alpha_\cau/\kappa_\cau>\alpha_\id/\kappa_\id$ for all $1\le L<T$, with the ratio increasing roughly like $\log T$ when $L=1$---the regime where causal attention is most effective.

\section{Low-Rank-plus-Noise Structure of Pre-trained Embeddings}\label{app:gpt2}

This appendix empirically supports the modelling choice made in Section~\ref{sec:embed}, namely that the embedding table $\mE$ be written as the sum of a low-rank ``signal'' component and a noise term. We illustrate that the singular-value spectrum of a large pre-trained word-embedding matrix does exhibit this structure: a bulk density (noise) plus a small number of outliers (the signal).

\paragraph{Setup.}
We inspect the token-embedding matrix $\mE\in\sR^{V\times d}$ of GPT-2~\citep{radford2019language} with $V=50257$ and $d=768$. The matrix is centred column-wise and we compute the eigenvalues of the $d\times d$ Gram matrix $\mE^{\top}\mE/V$, whose empirical distribution matches the Silverstein law \citep{silverstein1995empirical} (i.e., generalization of the Marchenko-Pastur law for non-isotropic covariance profiles) up to a finite number of outlier eigenvalues corresponding to the rank of the ``signal'' component.

\begin{figure}[h]
    \centering
    \includegraphics[width=0.98\linewidth]{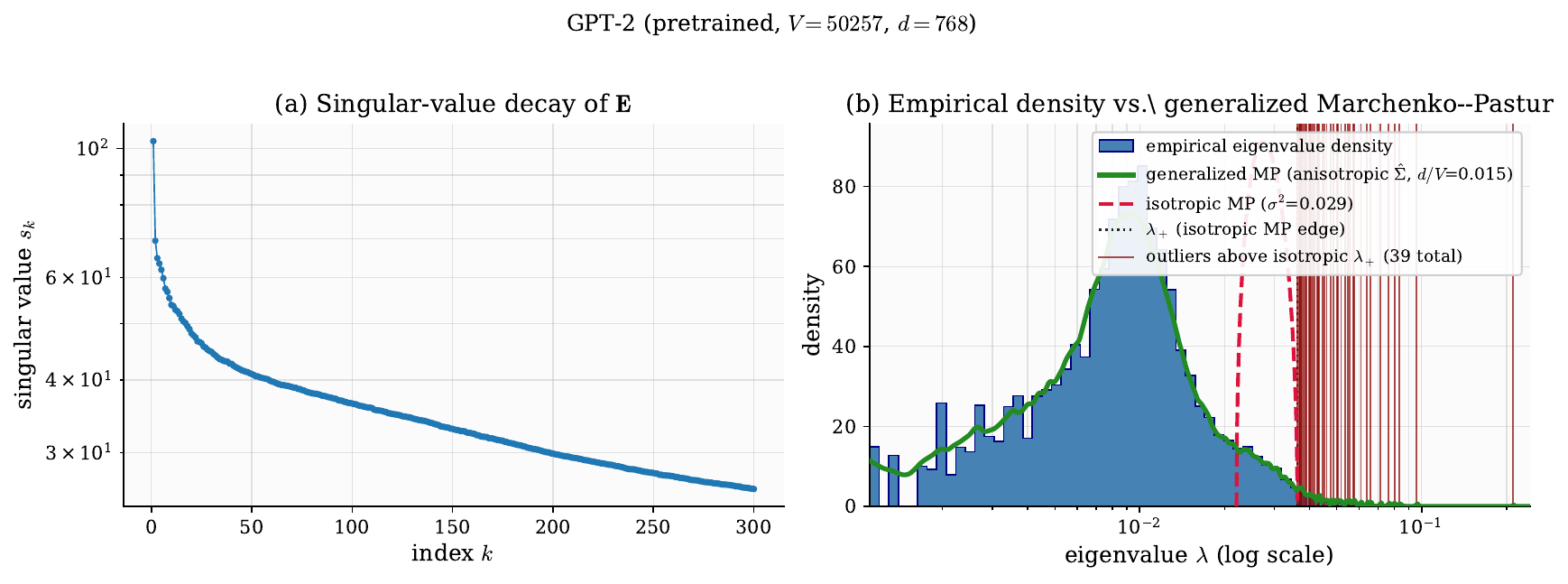}
    \caption{Empirical spectrum of the centred GPT-2 token-embedding matrix ($V=50257$, $d=768$). \textbf{(a)}~Singular-value decay $s_k$ of $\mE$ on a log scale: a handful of outliers ($\sim 10$--$30$) separates cleanly from a rapidly flattening bulk. \textbf{(b)}~Histogram of the bulk eigenvalues of $\mE^{\top}\mE/V$ together with the Silverstein density \citep{silverstein1995empirical} with covariance matrix estimated from the raw embedding table. The empirical bulk follows its theoretical counterpart closely; the residual outlier eigenvalues above the right bulk edge $\lambda_+$ are the spectral footprint of the low-rank signal.}
    \label{fig:gpt2_spectrum}
\end{figure}

\paragraph{Takeaway.}
The separation observed in Figure~\ref{fig:gpt2_spectrum}(a), together with the near-exact agreement between the empirical bulk and the theoretical prediction in Figure~\ref{fig:gpt2_spectrum}(b), is consistent with several prior empirical studies of pre-trained embeddings~\citep{mu2017all,ethayarajh2019contextual,gao2019representation}: the embedding matrix is well described by a ``top-$r$ plus isotropic noise'' model with $r \ll d$. Our analytical model~\eqref{eq:embedding} takes the minimal-rank instance ($r=1$) with isotropic covariance, which is the simplest setting that already yields a nontrivial signal-recovery phase diagram. The results in Remark~\ref{rem:universality} further imply that the Gaussianity of the noise component is inessential: the spectral conclusions of Theorems~\ref{thm:bulk}--\ref{thm:sample_overlap} extend to any noise distribution with a bounded fourth moment, matching the empirical generality of the structure observed above.

\section{Downstream Classification}\label{appendix:classif}

We provide here the settings of the downstream classification experiment reported in Figure~\ref{fig:classification} of the main paper. We consider the following parameters:
\begin{itemize}
    \item prefix correlation model $\mR = \vone_L\vone_L^\top \oplus \mI_{T-L}$,
    \item dimensions $d=300$, $V=500$, $N=800$, $T=10$, $L=3$,
    \item binary labels $y_n = \mathrm{sign}(\vone_L^\top \xi_n)\in\{\pm 1\}$,
    \item $80\%/20\%$ train/test split,
    \item L2 regularization coefficient $\lambda=1$,
    \item signal-strength sweep $\|\vmu\|\in[0,5]$.
\end{itemize}

The different considered pooling methods are:
\begin{enumerate}
    \item[(i)]   \textbf{Mean pooling:} $\vw^{\id} = \vone_T/T$.
    \item[(ii)]  \textbf{Causal attention:} $\vw^{\cau}$ is the
                 deterministic harmonic weight vector of Proposition~\ref{prop:harmonic}
                 (score-free limit of causal self-attention with
                 $\tau/d$ scaling).
    \item[(iii)] \textbf{Optimal $\vw^{\opt}$:} top eigenvector of $\mR$,
                 normalized so that $\vone^\top \vw^{\opt}=1$
                 (Theorem~\ref{thm:optimal}).
    \item[(iv)] \textbf{Learned $\vw$ (joint L2):} jointly minimize
    \begin{equation}\label{eq:joint_l2}
        \min_{\vw\in\Delta_T,\; \vbeta\in\sR^d}
        \;\frac{1}{N_\mathrm{tr}}\bigl\| \vy_\mathrm{tr}
            - \mC(\vw)\,\vbeta \bigr\|_2^2
        \;+\; \lambda\,\|\vbeta\|_2^2,
    \end{equation}
    where $\mC(\vw)\in\sR^{N_\mathrm{tr}\times d}$ has rows
    $\vc_n = \sum_{t=1}^T w_t X_t^{(n)}$ and
    $\Delta_T = \{\vw\in\sR_+^T : \vone^\top\vw = 1\}$ is the
    probability simplex.
\end{enumerate}

\paragraph{Parameterization and optimizer.}
To enforce $\vw\in\Delta_T$, we reparameterize $\vw = \softmax(\vphi)$ for
$\vphi\in\sR^T$. For any fixed $\vw$, the objective~\eqref{eq:joint_l2} is
a ridge regression in $\vbeta$ with closed-form solution
\[
    \vbeta^\star(\vw)
    = \bigl(\mC(\vw)^\top\mC(\vw) + \lambda N_\mathrm{tr}\,\mI_d\bigr)^{-1}
      \mC(\vw)^\top \vy_\mathrm{tr}.
\]
Substituting $\vbeta^\star(\vw)$ back yields a function
$g(\vphi) = g(\vw(\vphi))$ whose gradient with respect to $\vphi$ follows
from the envelope theorem applied at the ridge optimum:
\[
    \frac{\partial g}{\partial \mC}
    = -\frac{2}{N_\mathrm{tr}}\bigl(\vy_\mathrm{tr}
        - \mC\,\vbeta^\star\bigr)\,{\vbeta^\star}^\top,
    \qquad
    \frac{\partial g}{\partial w_t}
    = \sum_{n,j}
      \frac{\partial g}{\partial C_{n,j}}\, X_{t,j}^{(n)},
\]
which is then pushed through the softmax as
$\partial g/\partial \vphi = \vw\odot\bigl(\nabla_{\vw}g
- \langle \nabla_{\vw}g,\vw\rangle\bigr)$. The resulting
low-dimensional problem in $\vphi\in\sR^{T}$ is solved by L-BFGS-B with
multiple random restarts; the best-loss initialization is retained.